\documentclass[10pt,twocolumn,letterpaper]{article}

\usepackage[pagenumbers]{cvpr} 

\usepackage{url}
\usepackage{algorithm2e}
\usepackage{tabularx}
\usepackage{graphicx}
\usepackage{makecell}
\usepackage{amsmath}
\usepackage{amsthm}
\usepackage{amssymb}
\usepackage{pifont}
\usepackage{wrapfig}
\usepackage{xcolor}
\usepackage{color, colortbl}

\newtheorem{theorem}{\textit{Theorem}}

\definecolor{brightgreen}{rgb}{0.4, 1.0, 0.0}
\definecolor{LightCyan}{rgb}{0.88,1,1}
\definecolor{cadetgrey}{rgb}{0.57, 0.64, 0.69}
\definecolor{carolinablue}{rgb}{0.6, 0.73, 0.89}
\definecolor{Gray}{gray}{0.9}
\newcommand*\colourcheck[1]{%
  \expandafter\newcommand\csname #1check\endcsname{\textcolor{#1}{\ding{52}}}%
}
\newcommand*\colourx[1]{%
  \expandafter\newcommand\csname #1x\endcsname{\textcolor{#1}{\ding{55}}}%
}
\usepackage{amsfonts}

\colourcheck{blue}
\colourcheck{green}
\colourcheck{red}
\colourcheck{brightgreen}
\colourcheck{black}
\colourcheck{gray}
\colourx{blue}
\colourx{green}
\colourx{red}
\colourx{black}
\colourx{gray}

\newcommand{\graytext}[1]{\textcolor{gray}{#1}}

\usepackage{multirow}
\usepackage{enumerate}
\usepackage{color}
\usepackage{xspace}
\usepackage{booktabs}
\usepackage{makecell}

\providecommand{\mypara}[1]{{\noindent{\bf #1}}}
\usepackage{enumitem}
\usepackage[pagebackref=true,breaklinks=true,colorlinks,bookmarks=false]{hyperref} 

\usepackage{defs}
\usepackage{spverbatim}
\usepackage{caption}

\newcommand{\MethodName}{\textit{Promptable Behaviors}\xspace}

\def\paperID{526} 
\def\confName{CVPR}
\def\confYear{2024}

\title{Promptable Behaviors:\\ Personalizing Multi-Objective Rewards from Human Preferences\vspace{-0.5cm}}

\author{\textbf{Minyoung Hwang$^{1}$, Luca Weihs$^{1}$, Chanwoo Park$^{2}$, Kimin Lee$^{3}$,}\\ \textbf{Aniruddha Kembhavi$^{1}$, Kiana Ehsani$^{1}$}
\\
$^{1}$PRIOR @ Allen Institute for AI, $^{2}$Massachusetts Institute of Technology, \\
$^{3}$Korea Advanced Institute of Science and Technology
}

\begin{document}
\maketitle

\begin{abstract}

Customizing robotic behaviors to be aligned with diverse human preferences is an underexplored challenge in the field of embodied AI.
In this paper, we present \MethodName, a novel framework that facilitates efficient personalization of robotic agents to diverse human preferences in complex environments. We use multi-objective reinforcement learning to train a single policy adaptable to a broad spectrum of preferences. We introduce three distinct methods to infer human preferences by leveraging different types of interactions: (1) human demonstrations, (2) preference feedback on trajectory comparisons, and (3) language instructions. We evaluate the proposed method in personalized object-goal navigation and flee navigation tasks in ProcTHOR~\cite{deitke2022️procthor} and RoboTHOR~\cite{deitke2020robothor}, demonstrating the ability to prompt agent behaviors to satisfy human preferences in various scenarios.\\ Project page: \href{https://promptable-behaviors.github.io}{https://promptable-behaviors.github.io}
\vspace{-0.15cm}
\end{abstract}

\section{Introduction}\label{sec:intro}
\begin{figure}[t!]{
\centering
\begin{center}
\includegraphics[width=0.9\linewidth]{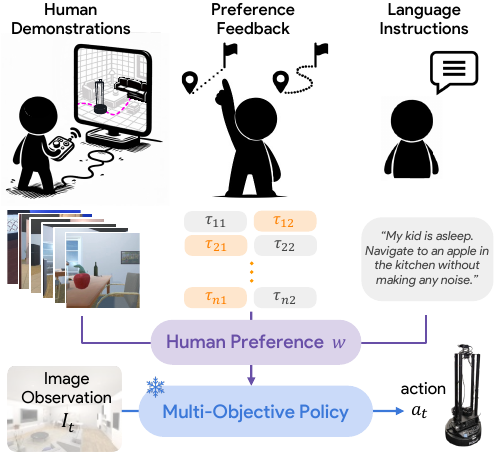}
\end{center}
\centering}\vspace{-0.3cm}
\caption{\textbf{Overview.}
\textit{Promptable Behaviors} captures human preferences across multiple objectives. We first train a multi-objective policy conditioned on the reward weight vector. After training, we freeze the policy and humans can provide their preferences with a wide range of options: (1) human demonstrations, (2) preference feedback on trajectory comparisons, and (3) language instructions.
} \label{fig:overview}\vspace{-0.15cm}
\end{figure}

Imagine a robot navigating in a house at midnight, asked to find an object without disturbing a child who just fell asleep. 
The robot is required to explore the house thoroughly in order to find the target object, but not collide with any objects to avoid making unnecessary noise. 
In contrast to this \textit{Quiet Operation} scenario, in the \textit{Urgent} scenario, a user is in a hurry and expects a robot to find the target object quickly rather than avoiding collisions.
These contrasting scenarios highlight the need for customizing robot policies to adapt to diverse and specific human preferences.

Although learning-based approaches~\cite{singh2022robot-learning-survey-rl, xiao2022robot-learning-survey-navigation} have significantly advanced the capability of robots to solve numerous tasks successfully, using these methods to customize robots for diverse human preferences remains a challenge~\cite{hellou2021personalization}. Common practices in embodied AI~\cite{deitke2022️procthor,deitke2020robothor} use reinforcement learning with a reward function designed for specific agent behaviors. 
However, hand-crafting a reward function by human experts is time-consuming and difficult for agents with complex dynamics and large state and action spaces.
To simplify the reward design process for non-expert users, recent methods~\cite{myers2022learning, hejna2023few, hwang2023rlhf} intuitively acquire reward models from human feedback. Yet, they have shortcomings in dealing with diverse preferences, since the agent has to be re-trained for each unique human preference.

We propose \MethodName, a novel personalization framework that deals with diverse human preferences without re-training the agent. The key idea of our method is to use \textit{multi-objective reinforcement learning} (MORL) as the backbone of personalizing a reward model.
We take a modular approach: (1) training a policy conditioned on a reward weight vector across multiple objectives and (2) inferring the reward weight vector that aligns with the user's preference. 
Using MORL, agent behaviors become promptable through adjustments in the reward weight vector during inference, without any policy fine-tuning. This significantly simplifies customizing robot behaviors to inferring a low-dimensional reward weight vector.

We provide a variety of options for users to provide their preferences to the agent.
Specifically, we introduce three distinct methods of reward weight prediction, see Figure~\ref{fig:overview}, leveraging different types of interaction: (1) human demonstrations, (2) preference feedback on trajectory comparisons, and (3) language instructions. 
Given human demonstrations in simulated environments, the agent can extract preferences from user-specific behaviors. Preferences could also be inferred from binary feedback on trajectory comparisons, enabling users to evaluate and contrast different agent behaviors rather than providing direct demonstrations. 
Finally, we utilize large-language models (LLMs) to translate language instructions into reward weight vectors. Using LLMs, even indirect or implicit instructions can be effectively interpreted based on extensive world knowledge.

We demonstrate \textit{Promptable Behaviors} in two personalized navigation tasks, object-goal navigation and flee navigation, in two environments, ProcTHOR~\cite{deitke2022️procthor} and RoboTHOR~\cite{deitke2020robothor}. 
Experimental results show that the agent behavior can be effectively prompted in both tasks. While the three reward weight prediction methods have their own advantages, preference feedback on trajectory comparisons shows the highest performance. In particular, our human evaluations demonstrate the effectiveness of our method. %

In summary, our main contributions include:
\begin{itemize}
    \item A novel framework for personalized learning that enables robots to align with diverse human preferences in complex embodied AI tasks without any policy fine-tuning.
    \item Three methods for inferring human preferences using human demonstrations, preference feedback on trajectory comparisons, and language instructions, each offering unique advantages.
    \item Demonstrations in two long-horizon personalized navigation tasks shows the effectiveness of our approach in prompting agent behaviors to satisfy human preferences.
\end{itemize}

\section{Related Work}
\subsection{Multi-Objective Reinforcement Learning}
Existing MORL algorithms are categorized into two main types~\cite{hayes2022morl-survey}: single-policy and multi-policy. \textit{Multi-policy} methods~\cite{mossalam2016deepmorl, alegre2023gpi-ls-pd, yang2019generalized-envQ, lu2022multi-CAPQL, xu2020prediction-PGMORL, reymond2022pareto-PCN, van2014multi-ParetoQ, roijers2016multi, cai2023distributional} train multiple policies, each corresponding with a single combination of objectives. 
However, in complex environments, training separate policies for each objective combination can be inefficient and resource-intensive.
On the other hand, \textit{single-policy} methods~\cite{van2013scalarized-MOQL, roijers2018multi, pan2020additional, siddique2020learning, peschl2021moral} transform the multi-objective problem into a single-objective problem through reward scalarization. ~\cite{siddique2020learning} present multi-objective forms of existing RL algorithms (e.g., PPO~\cite{schulman2017proximal} and A2C~\cite{mnih2016asynchronous}) and focuses on learning a single policy conditioned on the combination of multiple objectives.
While previous work show success in simple environments~\cite{Alegre+2022mo-gym}, tasks and objectives are often unrealistic. Recent work~\cite{cheng2023multi-crowd-aware, deshpande2021navigation-ped-aware} apply MORL on collision-aware navigation tasks but use the true state of the environment or map the environment using a high-cost sensor.
We demonstrate single-policy MORL in complex, realistic robotic tasks, utilizing high-dimensional observations. Ask4Help~\cite{singh2022ask4help} shows that we can condition a policy on user preferences during training, but adjusts weight on a single dimension, while ours deals with at least three objectives.

\subsection{Learning from Demonstrations}
Given expert demonstrations, imitation learning (IL) and inverse-reinforcement learning (IRL) are the two prominent methods that guide agents to perform tasks by replicating and understanding observed behaviors, respectively. IL~\cite{ramrakhya2023pirlnav, hwang2023meta, ramrakhya2022habitat-web, zhu2021vln-soon} typically requires a large amount of high-quality expert data~\cite{ramrakhya2023pirlnav}. IRL aims to understand the underlying reward functions motivating expert behaviors rather than just copying the observed behaviors~\cite{arora2021inverserl-survey, adams2022inverserl-survey, fang2021visual-inverserl}, but also requires a sufficient amount of demonstrations and the learned reward can be overfit to the collected data. Compared to IL and IRL, our method requires significantly fewer demonstrations to make the agent behavior satisfy the user's preference because the agent can efficiently generalize to diverse preferences without any policy fine-tuning.

\subsection{Learning from Human Feedback}
Learning from human feedback such as ratings, rankings, or expert interventions has been studied in numerous prior works~\cite{akrour2011preference, akrour2012april, wilson2012bayesian, sugiyama2012preference, daniel2015active, el2016score, wirth2016model, furnkranz2012preference, akrour2014programming, ren2023askforhelp, christiano2017deep}.
Extending \cite{christiano2017deep} which uses preference feedback on pairwise trajectory comparisons, recent developments~\cite{sadigh2017active, biyik2018batch, biyik2019green, myers2022learning, lee2021pebble, hejna2023few, park2022surf, liang2022reward, liumetarewardnet, hejna2023inverse, hejna2023contrastive, hwang2023rlhf} have enhanced sample and feedback efficiency.
While these methods have shown promise in natural language processing~\cite{ouyang2022instructgpt} and simplified settings in robotics with low-dimensional state and action spaces, their scalability to more complex, long-horizon, robotic tasks remains underexplored. We show that training a multi-objective policy with MORL and subsequently optimizing the reward weight vector using human feedback greatly enhances \textit{the efficiency of handling diverse preferences}, especially in complex and long-horizon tasks. Additionally, we introduce group pairwise comparison, which significantly reduces the labeling effort compared to conventional methods by allowing users to compare groups of trajectories.

\section{Method}
We propose a novel framework, \MethodName, for personalized robot learning. \MethodName is an adaptable policy that can update its behavior to various objectives and user preferences. Our method is divided into two primary components: (1) training a promptable multi-objective policy, and (2) capturing the agent's desired behavior through interactions. The overview of the proposed method is illustrated in Figure~\ref{fig:overview}. Our multi-objective policy is adapted to individual users by adjusting the reward weights without any policy fine-tuning. For instance, suppose we wish to find reward weight vectors that align with the \textit{Quiet Operation} and \textit{Urgent} scenarios introduced in Section~\ref{sec:intro}. For the \textit{Quiet Operation} scenario, it is desirable to give high weight to safety like $[0,0.3,0.7]$ where the dimensions correspond to time efficiency, house exploration, and safety, respectively. On the other hand, reward weights such as $[1,0,0]$ would be aligned with the \textit{Urgent} scenario.

\begin{figure*}[t!]{\centering\includegraphics[width=1\linewidth]{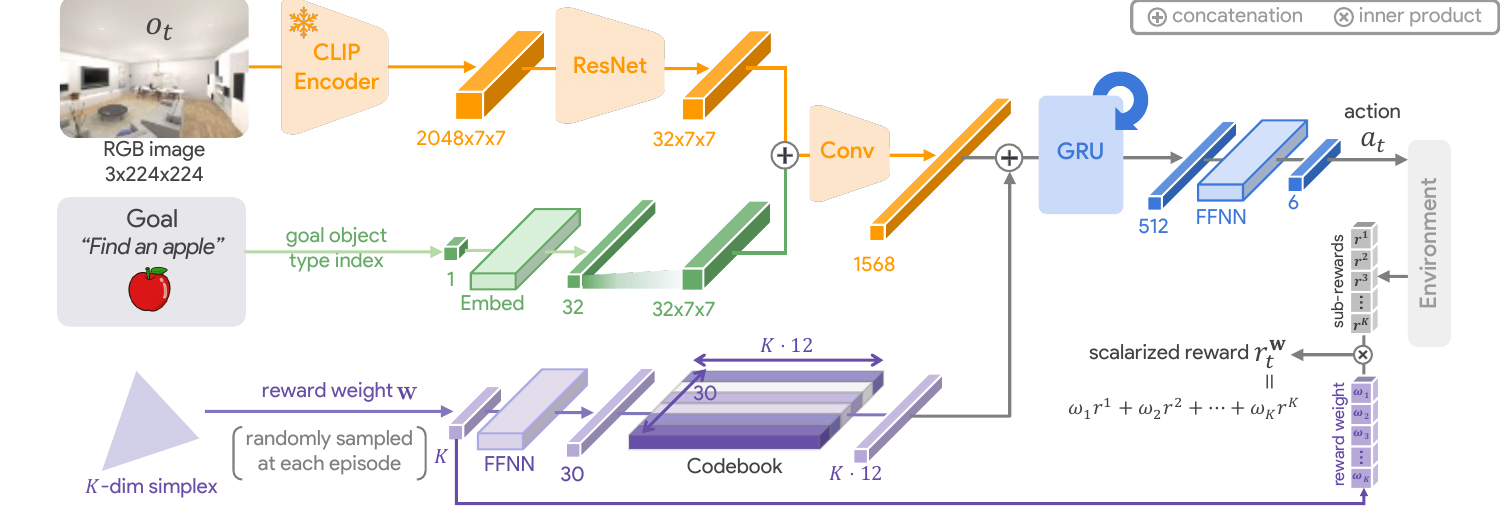}\centering}
\caption{\textbf{Network Architecture.} The figure illustrates a single-policy MORL architecture for object-goal navigation. A CLIP encoder alongside a ResNet processes the visual observation. This image embedding, concatenated with the goal embedding, is then fed into a convolutional layer and concatenated with the reward weight vector encoded via codebook, forming the input for the recurrent policy. The policy is trained to be multi-objective, modulated by the reward weight vector $\mathbf{w}$. During training the policy, the agent receives the reward as the weighted sum of sub-rewards from $K$ objectives, determined by the reward weight vector $\mathbf{w}$.
} \label{fig:network-architecture}\vspace{-0.3cm}
\end{figure*}

\subsection{Problem Formulation}
We solve two navigation tasks using a robotic agent: object-goal navigation and flee navigation. We follow the task definition of object-goal navigation in previous work~\cite{deitke2022️procthor, deitke2020robothor}, where the agent has to find an object of a given object category and execute an explicit \texttt{Done} action when the object is within 1m and visible in the agent's camera. The agent is allowed a maximum time horizon of $T=500$. Flee navigation requires the agent to maximize its distance from the initial location. This task is useful when the robot has to autonomously relocate to a distant location in the house through spatial reasoning. In both tasks, the agent uses an RGB image observation $o_t$ and outputs an action $a_t$ at time $t$. The action is chosen from \texttt{[MoveAhead, RotateRight, RotateLeft, Done, LookUp, LookDown]}. The agent state $s_t$ at time $t$ is set as $o_{1:t}$. Most of the previous work aim to improve general performance such as success rate and success weighted by path length (SPL)~\cite{deitke2022️procthor, deitke2020robothor, khandelwal2022simpleembclip}. In this paper, we open a new paradigm to consider the agent's behavior beyond rapid task completion. Such agent behavior can be measured and categorized as the sub-rewards on $K$ objectives, where each objective reflects a fundamental aspect of navigation. Detailed definitions of the objectives and the evaluation metrics in each task are provided in Section~\ref{sec:experiment-settings}.

\subsection{Promptable Multi-Objective Policy Training}
Our approach aims to develop a promptable and efficient framework for embodied AI tasks. Contrary to traditional RL methods that require a significant amount of time and resources to optimize for a single combination of different objectives, our method focuses on training a policy that can handle any linear combination of different objectives at test time. This reduces the dependency of the agent's behavior on the reward design choice of the practitioner who trains the policy. A na\"{i}ve approach is to apply multi-policy MORL, train multiple policies with different reward configurations, and choose the most appropriate policy output for inference. This divides the multi-objective problem into a series of single-objective problems and ensures the agent optimizes the policies on each combination. However, this training process is inefficient and cannot cover the entire set of combinations because the size of the weight space increases exponentially with the number of objectives.

Inspired by Ask4Help~\cite{singh2022ask4help} that trains a policy conditioned on the reward configuration, we condition our agent's policy on randomly sampled reward configuration during training and allow the agent to adapt to the user at inference time without additional training. While Ask4Help focuses on changing the reward configuration in a single dimension, we implement single-policy MORL~\cite{siddique2020learning} to handle multiple objectives.
Figure~\ref{fig:network-architecture} illustrates the network architecture of the policy in our method. We train a single policy with a scalarized reward function $r^{\mathbf{w}}=\mathbf{w}^\intercal \mathbf{r}$, which combines multiple objectives with a reward weight vector $\mathbf{w}$ randomly sampled from a $K$-dim simplex $\Delta_K=\{\mathbf{w}\in\mathbb{R}_+^K|\ ||\mathbf{w}||_1=1\}$. While most RL frameworks in embodied AI have a pre-defined and fixed $\mathbf{w}$, our policy is conditioned on the combination itself and explores various combinations during training. This makes the trained policy adaptable to various human preferences through the adjustment of the reward weight vector without any additional policy training.

\noindent\textbf{Visual Encoder Using CLIP.} Recent work \cite{khandelwal2022simpleembclip} has shown the strength of visual backbones of CLIP~\cite{radford2021clip} in embodied AI tasks. As described in Figure~\ref{fig:network-architecture}, we use a pre-trained CLIP ResNet-50 to encode $3\times224\times224$ RGB image into a $2048\times7\times7$ tensor. Since the pretrained model has shown its effectiveness in various visual navigation tasks, we freeze the weights of the encoder while training the policy. The CLIP embedding is merged with a $32$-dim goal embedding resulting in a shape of $64\times7\times7$. The concatenated tensor is passed into a CNN and is flattened into a $1568$-dim feature.

\noindent\textbf{Reward Weight Encoder.} Since the goal of the proposed method is to handle various combinations of objectives with a single policy, we randomly sample reward weights during training. At each episode, a reward weight vector $\mathbf{w}$ is uniformly sampled from a $K$-dim simplex $\Delta_K$ and the agent calculates the scalarized rewards based on $\mathbf{w}$ throughout the whole episode. 
Bringing insight from the recent success of using codebook as an effective representation module~\cite{eftekhar2023selective}, we use a feed-forward neural network (FFNN) to expand the dimension of $\mathbf{w}$ to $30$ and then pass it through a codebook with $30$ learnable $K{\cdot}12$-dim latent codes. This makes the policy handle unseen reward combinations using the learned codes. We also compare this method with an encoding approach extended from \cite{singh2022ask4help}, where an integer weight vector is encoded using a lookup table.

\noindent\textbf{Navigation Policy.}
We implement a multi-objective version of DD-PPO~\cite{wijmans2019ddppo, AllenAct} to maximize the expected reward. A fundamental difference compared to the traditional DD-PPO is that the policy $\pi$ is conditioned on the reward weight vector $\mathbf{w}$, which is randomly sampled at each episode. The agent calculates the rewards for multiple objectives at each timestep, and updates the policy based on the scalarized rewards. The RL loss tries to maximize the expected return averaged among different $\mathbf{w}$ and episodes.

\subsection{Reward Weight Prediction via Interaction}\label{subsec:reward-weight-prediction}
Effectively aligning agent behavior with human preferences is a key challenge in our work. The proposed framework focuses on predicting the optimal reward weight vector representing human preferences, based on different forms of interactions. As illustrated in Figure~\ref{fig:overview}, we explore three distinct interactions: (1) human demonstrations, (2) preference feedback on trajectory comparisons, and (3) language instructions. Each method offers a unique perspective and mechanism for capturing human preferences, thereby handling a diverse range of scenarios and user interactions.
Following the general context in MORL~\cite{hayes2022morl-survey}, we assume that human preferences remain constant over time and each human preference is captured through a linear combination of multiple objectives in the environment. Under these assumptions, a human user's true preference is represented as a reward weight vector $\mathbf{w}\in\Delta_K$.

\subsubsection{Human Demonstrations}
Getting human demonstrations is a direct and intuitive way for humans to express their preferences.
Given a demonstration $\tau_{h}=(s_1, a_1, ..., s_{T_h}, a_{T_h})$ $(T_h\leq T)$ from a human user, we can infer the user's inherent values and priorities. We identify the reward weight vector that most accurately reflects these preferences by maximizing the expected log-likelihood between the demonstrated action and the action distribution from the policy $\pi$ conditioned on the reward weight $\fw$. The weight prediction loss is defined as follows:\vspace{-0.5cm}

\begin{align}\label{eq:weight-prediction-demonstration}
    \mathcal{L}_{demo}(\mathbf{w};\tau_h) = -\sum\nolimits_{t=1}^{T_h}\log \pi(a_t\mid s_t\ ;\mathbf{w}).
\end{align}

\noindent For the optimization process, we use multiple initialization weights, including a uniform vector $[1/K, ..., 1/K]$. We then apply gradient descent from each of the initialization weights until the loss converges. By averaging the results from diverse initial conditions, we enhance the robustness and reliability of our weight prediction. Also, our method is faster than traditional IRL approaches since we only optimize $K$ parameters, with $K$ being at most $5$ in our setting.

\subsubsection{Preference Feedback on Trajectory Comparisons}
We further develop weight prediction methods that take human feedback, in the form of comparisons between agent trajectories, as input. We also propose a novel trajectory comparison method called group trajectory comparison, which is more feedback-efficient than conventional pairwise comparison~\cite{christiano2017deep}.

\noindent\textbf{Pairwise Trajectory Comparison.} In pairwise trajectory comparison, the human user is asked to select a trajectory that better aligns with their preference from a pair of trajectories. Given $N$ trajectory pairs, the user will provide $N$ binary preference labels. We denote the preference data as a set of trajectory pairs and preference labels $\mathcal{S}=\{(\tau_{i1}, \tau_{i2}, y_i)| 1\leq i\leq N\}$, where $\tau_{i1}$ and $\tau_{i2}$ are the two trajectories that the user observes at the $i^{th}$ trajectory comparison, $y_i=1$ indicates $\tau_{i1}$ is preferred to $\tau_{i2}$, and $y_i=0$ indicates otherwise. Notably, we do not consider the ties and provide the human user the ability to skip indistinguishable queries.
We use a common assumption in preference-based learning~\cite{wirth2017pbrl-survey} that given a reward weight vector $\mathbf{w}$ that reflects human preference, the user chooses trajectory $\tau_1$ to be preferred over $\tau_2$ with a preference probability based on the Bradley-Terry model~\cite{bradley1952rank} as follows:\vspace{-0.2cm}

\begin{align*}\small
    P(\tau_1\succ \tau_2; \mathbf{w}) = \frac{\exp(\mathbf{w}^\intercal \mathbf{r}(\tau_1))}{\exp(\mathbf{w}^\intercal \mathbf{r}(\tau_1)) + \exp(\mathbf{w}^\intercal \mathbf{r}(\tau_2))},
\end{align*}

\noindent where $\tau_{1}\succ \tau_{2}$ denotes $\tau_{1}$ is preferred to $\tau_{2}$ and $\mathbf{r}(\tau)=\sum_{t=0}^{T} \mathbf{r}(s_t, a_t)$.
The model handles the inherent stochasticity and inconsistency in human preferences, rather than assuming human preferences are deterministic and perfectly rational. We solve an optimization problem that maximizes the expected log-likelihood of preferences in pairwise trajectory comparisons as follows:\vspace{-0.35cm}

{\small
\begin{align*}
    \max_{\mathbf{w}\in\Delta_K} \mathop{\mathbb{E}}_{(\tau_1, \tau_2, y)\in \mathcal{S}}\!\!\!\!\!\log\left(y P(\tau_1\!\succ\! \tau_2;\fw) \!+\! (1\!-\!y) P(\tau_1\!\prec\! \tau_2;\fw)\right).
\end{align*}
}

\noindent\textbf{Group Trajectory Comparison.}
We also introduce a novel trajectory comparison method called group trajectory comparison, where the human user observes groups of trajectories instead of individual pairs. For instance, consider comparing a group emphasizing safety against another set, prioritizing path efficiency. The difference between the two groups becomes clearer, making the process of giving preference labels more straightforward for users.
At each iteration, we sample two groups of $M$ trajectories, each trajectory group generated with the same reward weights, while different groups have different reward weights. We ask the user to compare the groups of trajectories and provide a preference label. Each group comparison yields an inequality constraint, filtering out a volume of the reward weight space less likely to match the user's preferences. We repeat this process for $N$ iterations with different groups of trajectories and different reward weights. By adding $N$ constraints and performing constrained optimization, we effectively narrow down the search area for the most probable reward weight vector. In Section~\ref{subsec:exp-reward-weight-prediction}, we show that group comparison is $17.8\%$ more effective than conventional pairwise trajectory comparison in human evaluation. It is also less ambiguous for the human user to compare groups that have distinct differences. Detailed theoretical analyses are included in the supplementary material.

\subsubsection{Language Instructions}
We leverage the power of LLMs to interpret language instructions and quantify human preferences as numerical reward weights. LLMs can adapt to the nuances of user instructions, and the models can translate natural language instructions into reward weights. We ask ChatGPT~\cite{chatgpt} to output the optimal reward weight vector from a language instruction of the user given the task description and definitions of the objectives. We use in-context learning (ICL)~\cite{brown2020incontext} by providing six examples of instruction and answer pairs, collected from six human experts. We also apply chain-of-thought (CoT) reasoning~\cite{kojima2022chainofthought} to handle the complex, multi-step process of deciding reward weights on multiple, often conflicting, objectives. This method is highly beneficial in scenarios requiring rapid adaptation to user preferences without domain knowledge because LLMs can infer the importance to place on each objective based on the context in the instruction utilizing its world knowledge.

\section{Experiments}\label{sec:experiment} 
We evaluate our method on personalized object-goal navigation (ObjectNav) and flee navigation (FleeNav) in ProcTHOR~\cite{deitke2022️procthor} and RoboTHOR~\cite{deitke2020robothor}, environments in the AI2-THOR~\cite{kolve2017ai2thor} simulator. For ObjectNav, there are $16$ and $12$ target object categories in ProcTHOR and RoboTHOR, respectively. The policy is evaluated across various scenarios to ensure that it aligns with human preferences and achieves satisfactory performance in both tasks. We show that the proposed method effectively prompts agent behaviors by adjusting the reward weight vector and infers reward weights from human preferences using three distinct reward weight prediction methods.

\subsection{Experiment Settings}~\label{sec:experiment-settings}

\mypara{Training details.} We train our models using the AllenAct~\cite{AllenAct} reinforcement learning framework. In ProcTHOR ObjectNav, we train our policy for $130M$ steps over $10k$ houses and validate with $100$ episodes in $67$ unseen houses. In ProcTHOR FleeNav, we train for $50M$ steps over $10k$ houses and validate in $100$ episodes in $71$ unseen houses. In RoboTHOR ObjectNav and FleeNav, we train for $100M$ steps in $60$ houses and validate with $100$ episodes in $15$ unseen houses. All methods are trained with $8$ NVIDIA RTX A6000 GPUs and $80$ samplers. We use Adam optimizer with a learning rate of $0.0003$ for all training.

\noindent\textbf{Objectives.} In personalized object-goal navigation, we define five objectives: time efficiency, path efficiency, house exploration, object exploration, and safety. 
\textit{Time efficiency} is designed to encourage the agent to complete the episode quickly, while \textit{path efficiency} aims to find the target object via the shortest path. House exploration and object exploration encourage the agent to explore more at the expense of efficiency. \textit{House exploration }is designed to favor covering a larger area, while \textit{object exploration} aims to observe more objects within the agent's camera.
\textit{Safety} encourages the agent to avoid colliding with obstacles and visiting areas where you could get trapped or stuck. We provide the exact equations for calculating sub-rewards for each objective in the supplementary material.

In personalized flee navigation, there are three objectives: time efficiency, house exploration, and safety. We follow the definitions of the objectives in object-goal navigation. We do not consider preferences over success, since achieving the goal is a default expectation in both tasks.

\noindent \textbf{Baselines.} We use EmbCLIP~\cite{khandelwal2022simpleembclip} as the single-objective RL (SORL) baseline. Also, we implement a multi-policy MORL baseline, prioritized EmbCLIP, that trains $K$ policies separately, where each policy is trained with a fixed and spiked reward weight vector prioritizing one specific objective $\nu$ times more than the other objectives. 
We set $\nu$ as $4$ and $3$ for ObjectNav and FleeNav, respectively.

\noindent \textbf{Evaluation Metrics.} In ObjectNav, we evaluate the general performance of the agent using success rate, success weighted by path length (SPL), distance to goal, and episode length. An episode is recorded as a success if the agent executes a \texttt{Done} action within $1m$ of the target object and the object is visible in the agent's last observation. Success rate is measured as the number of succeeded episodes divided by the total number of episodes, $|E|$. SPL is calculated as $\frac{1}{|E|}\sum_{i=1}^{|E|}S_i\frac{\ell_i}{max(\ell^{min}_i, \ell_i)}$, where $S_i$, $\ell_i$, and $\ell^{min}_i$ denote the binary success value, path length, and the shortest path length at the $i^{th}$ episode. In FleeNav, we evaluate the agent using success rate, path length weighted by path length (PLOPL), distance to the farthest point, and episode length. Success at each episode is determined as $\ell/\ell^{max}$, where $\ell$ and $\ell^{max}$ denote the Euclidean distance from the initial point to the last point and the distance from the initial point to its farthest point, respectively. PLOPL is measured as the path length divided by the maximum path length at each episode. SPL and PLOPL consider the length of the trajectory, not only the initial and last positions of the agent.

To evaluate three reward weight prediction methods, we collect demonstrations and feedback from real human users for five different scenarios, each prioritizing one or two objectives in object-goal navigation. 
We perform human evaluations by calculating the win rate~\cite{jang2023personalized}, showing a pair of trajectories to the user and asking which trajectory is more preferred in each objective. The win rate is calculated as $\frac{1}{K(K-1)}\sum_{i=1}^K\sum_{j\neq i}H(\zeta_i, \tau_i, \tau_j)$, where $\tau_i$ is generated for the $i^{th}$ scenario $\zeta_i$ and $H(\zeta_i, \tau_i, \tau_j)$ is $1$ if $\tau_i\succ\tau_j$ in $\zeta_i$.
For instance, suppose we have generated $\tau_1$ and $\tau_2$ for the \textit{Quiet Operation} and \textit{Urgent} scenarios in Section~\ref{sec:intro}, respectively. Presenting the \textit{Quiet Operation} scenario with the two trajectories to the user, we get a score of $1$ when the user prefers $\tau_1$ in the given situation.

Furthermore, we measure the weight prediction performance as the cosine similarity between the estimated reward weight and the reward weight determined by human experts. We also measure the Generalized Gini Index (GGI)~\cite{weymark1981generalizedginiindex, busa2017multiobjective-gini} to statistically measure the disparity across multiple objectives in the predicted weights. A higher GGI indicates the weight vector is concentrated or peaked towards a few specific objectives.
More scenarios and experiment details are provided in the supplementary material.

In our experiments, prioritizing different objectives demonstrates diverse human preferences in everyday scenarios. By evaluating how the agent's behavior changes with varied objective prioritization, we gain insights into the adaptability of \MethodName to diverse preferences. We first show that the multi-objective policy outputs different agent behaviors based on objective prioritization. Then, we demonstrate how the agent can induce its behavior to satisfy user preference, given human demonstrations, preference feedback, and language instructions.

\subsection{Promptable Behaviors in Embodied AI}
\begin{table*}[ht!]
\resizebox{\textwidth}{!}{
\begin{tabular}{ccclccccccccc} 
\Xhline{2\arrayrulewidth}
{\textbf{Method}} & \textbf{Multi-Objective} & \multicolumn{2}{c}{\textbf{\begin{tabular}[c]{@{}c@{}}Prioritized\\ Objective\end{tabular}}} & \multicolumn{1}{c}{{\textbf{Success}}} & \multicolumn{1}{c}{{\textbf{SPL}}} & \multicolumn{1}{c}{{\textbf{\begin{tabular}[c]{@{}c@{}}Distance\\ to Goal\end{tabular}}}} & \multicolumn{1}{c}{{\textbf{\begin{tabular}[c]{@{}c@{}}Episode\\ Length\end{tabular}}}} & \multicolumn{5}{c}{\textbf{Sub Rewards $\uparrow$}} \\
\multicolumn{4}{c}{} & $\uparrow$ & $\uparrow$ & $\downarrow$ & $\downarrow$ & Time Efficiency & Path Efficiency & House Exploration & Object Exploration & Safety \\ \hline
EmbCLIP \citep{khandelwal2022simpleembclip} & \blackx & \multicolumn{1}{|c|}{\graytext{a}} & - & 0.611 & 0.455 & 1.677 & 105.389 & 0.767 & 0.581 & 0.703 & 0.731 & 0.556 \\
\hline
 \multirow{5}{*}{\begin{tabular}[c]{@{}c@{}}Prioritized\\ EmbCLIP\end{tabular}} & \multirow{5}{*}{\begin{tabular}[c]{@{}c@{}}Multi-Policy\end{tabular}} & \multicolumn{1}{|c|}{\graytext{b}} & Time Efficiency & 0.560 & 0.445 & 2.803 & 52.060 & \cellcolor{Gray}\textbf{0.926} & 0.317 & 0.136 & 0.247 & 0.746 \\
 &  & \multicolumn{1}{|c|}{\graytext{c}} & Path Efficiency & 0.611 & 0.449 & 2.038 & 106.444 & 0.764 & 0.515 & 0.590 & 0.731 & 0.693 \\
 &  & \multicolumn{1}{|c|}{\graytext{d}} & House Exploration & 0.200 & 0.113 & 3.921 & 350.960 & 0.033 & \cellcolor{Gray}\textbf{0.677} & \cellcolor{Gray}\textbf{2.868} & 0.161 & 0.012 \\
 &  & \multicolumn{1}{|c|}{\graytext{e}} & Object Exploration & 0.611 & 0.513 & 2.439 & 138.389 & 0.668 & 0.414 & 0.703 & \cellcolor{Gray}\textbf{0.731} & 0.556 \\ 
 &  & \multicolumn{1}{|c|}{\graytext{f}} & Safety & 0.480 & 0.391 & 3.237 & 56.620 & 0.912 & 0.016 & 0.130 & 0.004 & \cellcolor{Gray}\textbf{0.834} \\
\hline
\multirow{6}{*}{\begin{tabular}[c]{@{}c@{}}Promptable\\ Behaviors\\ (Ours)\end{tabular}} & \multirow{6}{*}{\begin{tabular}[c]{@{}c@{}}Single-Policy\end{tabular}} & \multicolumn{1}{|c|}{\graytext{g}} & - & 0.600 & 0.496 & 2.526 & 86.070 & 0.824 & 0.589 & 0.336 & 0.412 & 0.770 \\
 &  & \multicolumn{1}{|c|}{\graytext{h}} & Time Efficiency & 0.560 & 0.492 & 2.675 & 51.760 & \cellcolor{LightCyan}\textbf{0.927} & 0.375 & 0.078 & 0.301 & 0.772 \\
 &  & \multicolumn{1}{|c|}{\graytext{i}} & Path Efficiency & 0.650 & \textbf{0.543} & 2.213 & 115.350 & 0.737 & \cellcolor{LightCyan}\textbf{0.907} & 0.451 & 0.674 & 0.665 \\
 &  & \multicolumn{1}{|c|}{\graytext{j}} & House Exploration & {\textbf{0.680}} & 0.506 & 2.253 & 159.440 & 0.605 & 0.902 & \cellcolor{LightCyan}\textbf{0.995} & 0.705 & 0.563 \\
 &  & \multicolumn{1}{|c|}{\graytext{k}} & Object Exploration & 0.650 & 0.525 & 2.198 & 94.890 & 0.798 & 0.829 & 0.358 & \cellcolor{LightCyan}\textbf{0.725} & 0.754 \\
 &  & \multicolumn{1}{|c|}{\graytext{l}} & Safety & 0.500 & 0.446 & 2.875 & 51.890 & 0.927 & 0.211 & 0.083 & 0.096 & \cellcolor{LightCyan}\textbf{0.829} \\
\Xhline{2\arrayrulewidth}
\end{tabular}}\vspace{-0.18cm}
\caption{\textbf{Performance in ProcTHOR ObjectNav.} We evaluate each method in the validation set with six different configurations of objective prioritization: uniform reward weight across all objectives and prioritizing a single objective 4 times as much as other objectives. Sub-rewards for each objective are accumulated during each episode, averaged across episodes, and then normalized using the mean and variance calculated across all methods. Colored cells indicate the highest values in each sub-reward column.}\label{tab:procthor-objectnav}\vspace{-0.2cm}
\end{table*}

\begin{figure*}[t!]{\centering
    \subfloat[\centering Prioritizing Safety in ObjectNav]{
            \includegraphics[width=0.22\linewidth]{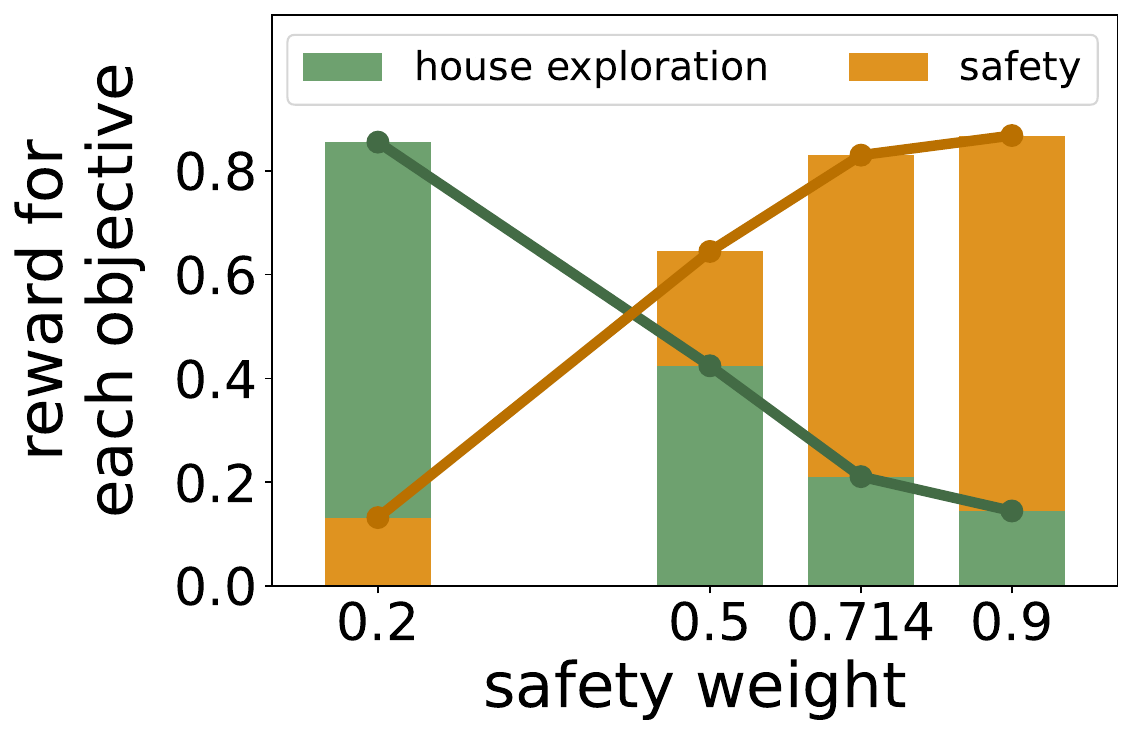}
            }\centering\hspace{0.5cm}
    \subfloat[\centering Trajectory Visualizations with Different Prioritizations on Objectives]{
            \includegraphics[width=0.72\linewidth]{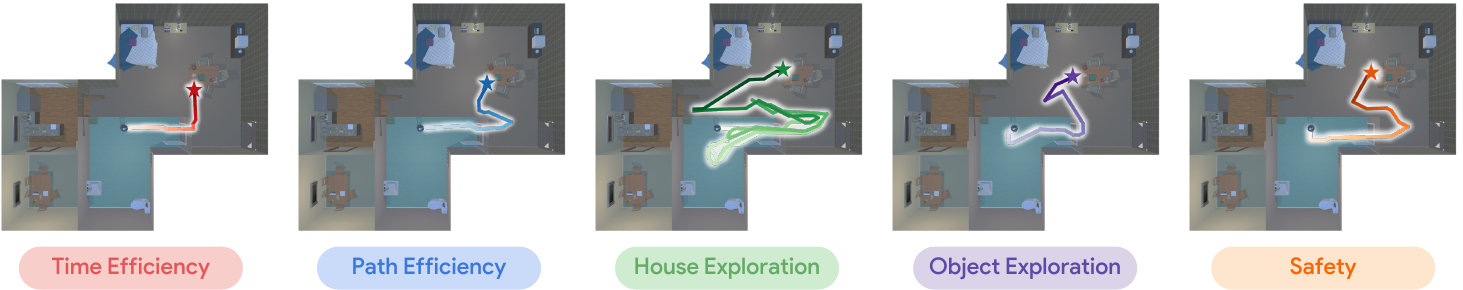}
            }\centering
    \vspace{-0.2cm}
    \caption{\small{\textbf{Prompting Agent Behaviors by Adjusting Reward Weights.} (a) As we prioritize safety more, the average safety reward increases while the average reward of a conflicting objective, house exploration, decreases. We normalize the rewards for each objective using the mean and variance calculated across all weights. (b) In each figure, agent trajectory is visualized when an objective is prioritized 10 times as much as other objectives. The agent's final location is illustrated as a star. 
    }
    }
    \label{fig:performance-graphs}
}\vspace{-0.1cm}
\end{figure*}

The first question we would like to answer is ``How well does \MethodName adapt its policy to reflect the changes in the input reward weights?". The goal is to show that our method can adapt its behavior to the reward weights during inference time more effectively than the baseline.
To test this hypothesis, we first measure our model's behavior given a reward weight vector prioritized on a single objective. For instance, the policy should achieve higher exploration when the exploration reward's weight is peaked than when no objective is prioritized or when the reward weight is peaked for safety. Thus, we evaluate \MethodName and Prioritized EmbCLIP across $K+1$ reward weights, including one uniform weight and $K$ peaked weights. For peaked reward weight vectors, we set the weight for the prioritized objective $\nu$ times greater than the weights for other objectives.
Among all methods, our approach most effectively reflects the prioritization of objectives in agent behaviors. Results in Table~\ref{tab:procthor-objectnav} and Table~\ref{tab:procthor-fleenav} show that the proposed method outperforms the baseline in ProcTHOR ObjectNav and FleeNav while showing efficiency in training, requiring much less computational resources compared to Prioritized EmbCLIP. Note that the performance of EmbCLIP does not perfectly match the performance reported in \cite{eftekhar2023selective} since we re-implement the method with our network architecture and evaluate it in a smaller validation set.
Results in the RoboTHOR environment and detailed analyses of all experiments are provided in the supplementary material.

\textbf{Our method achieves high success rates while efficiently optimizing the agent behavior for each objective.} In Table~\ref{tab:procthor-objectnav}, Prioritized EmbCLIP and \MethodName show different performance based on the prioritization of objectives. When prioritizing house exploration, both MORL methods achieve higher house exploration reward than the SORL baseline, EmbCLIP. When house exploration is prioritized, the proposed method shows the highest success rate (row j in Table~\ref{tab:procthor-objectnav}), $11.3\%$ higher than EmbCLIP, while Prioritized EmbCLIP shows the lowest success rate (row d in Table~\ref{tab:procthor-objectnav}) among all methods and reward configurations. Additionally, our method achieves the highest SPL and the path efficiency reward when path efficiency is prioritized (row i in Table~\ref{tab:procthor-objectnav}), outperforming EmbCLIP by $19.3\%$ and $56.1\%$, respectively. This implies that the proposed method effectively maintains general performance while satisfying the underlying preferences in various prioritizations. In contrast, Prioritized EmbCLIP fails to improve path efficiency reward when the corresponding objective is prioritized. This could be due to the design choice of $\nu$, which determines the sensitivity of the prioritized objective. Trying various $\nu$ might improve the alignment, but it is challenging to train the policy multiple times with different $\nu$. Selecting a proper $\nu$ is much easier in \MethodName because our policy has already observed random reward weight vectors during training. For FleeNav, Table~\ref{tab:procthor-fleenav} shows that both MORL methods successfully prompt agent behaviors through reward weight adjustments. Our method achieves success rates higher than $0.7$ in all cases, while Prioritized EmbCLIP shows a low success rate when time efficiency is prioritized (row a in Table~\ref{tab:procthor-fleenav}).

To check how conflicting objectives affect each other, we assess an experiment to evaluate the trained policy by adjusting the weight of the most prioritized objective from $0.2$ to $0.9$. Figure~\ref{fig:performance-graphs} (a) shows an example when we observe trade-offs between two conflicting objectives in ObjectNav: safety and house exploration. As we increase the weight for safety, the safety reward increases while the reward for its conflicting objective, house exploration, decreases. Figure~\ref{fig:performance-graphs} (b) visualizes five trajectories that prioritize four different objectives in the same episode. The difference between trajectories implies that prioritizing time efficiency or path efficiency encourages the agent to move through a shorter path while prioritizing house exploration or object exploration encourages the agent to explore the house more thoroughly. The safety reward column in Table~\ref{tab:procthor-objectnav} shows that the agent receives a higher safety reward when safety is prioritized, which means that the agent learns to avoid visiting narrow places or moving closely to near objects and walls.

\begin{table*}[t!]
\footnotesize
\resizebox{\textwidth}{!}{
\begin{tabular}{ccclccccccc}
\hline
{\textbf{Method}} & \textbf{Multi-Objective} & \multicolumn{2}{c}{\textbf{\begin{tabular}[c]{@{}c@{}}Prioritized\\ Objective\end{tabular}}} & \multicolumn{1}{c}{{\textbf{Success}}} & \multicolumn{1}{c}{{\textbf{PLOPL}}} & \multicolumn{1}{c}{{\textbf{\begin{tabular}[c]{@{}c@{}}Distance\\ to Furthest\end{tabular}}}} & \multicolumn{1}{c}{{\textbf{\begin{tabular}[c]{@{}c@{}}Episode\\ Length\end{tabular}}}} & \multicolumn{3}{c}{\textbf{Sub Rewards $\uparrow$}} \\
 &  &  &  & $\uparrow$ & $\uparrow$ & $\downarrow$ & $\downarrow$ & Time Efficiency & House Exploration & Safety\\ \hline
\multirow{3}{*}{\begin{tabular}[c]{@{}c@{}}Prioritized\\ EmbCLIP\end{tabular}} & \multirow{3}{*}{\begin{tabular}[c]{@{}c@{}}Multi-Policy\end{tabular}} & \multicolumn{1}{|c|}{\graytext{a}} & Time Efficiency & 0.691 & 0.810 & 7.360 & 57.090 & \cellcolor{Gray}\textbf{0.875} & 0.420 & 0.138 \\
 &  & \multicolumn{1}{|c|}{\graytext{b}} & House Exploration & \textbf{0.759} & 0.872 & 6.704 & 58.330 & 0.839 & \cellcolor{Gray}\textbf{0.835} & 0.215 \\
 &  & \multicolumn{1}{|c|}{\graytext{c}} & Safety & 0.723 & 0.856 & 7.391 & 57.640 & 0.859 & 0.676 & \cellcolor{Gray}\textbf{0.487} \\ \hline
\multirow{4}{*}{\begin{tabular}[c]{@{}c@{}}Promptable\\ Behaviors\\ (Ours)\end{tabular}} & \multirow{4}{*}{\begin{tabular}[c]{@{}c@{}}Single-Policy\end{tabular}} & \multicolumn{1}{|c|}{\graytext{d}} & - & 0.700 & 0.805 & 7.013 & 69.020 & 0.531 & 0.365 & 0.522 \\
 &  & \multicolumn{1}{|c|}{\graytext{e}} & Time Efficiency & 0.728 & 0.832 & 6.592 & 66.490 & \cellcolor{LightCyan}\textbf{0.604} & 0.434 & 0.563 \\
 &  & \multicolumn{1}{|c|}{\graytext{f}} & House Exploration & 0.737 & 0.861 & 6.317 & 71.500 & 0.460 & \cellcolor{LightCyan}\textbf{0.813} & 0.089 \\
 &  & \multicolumn{1}{|c|}{\graytext{g}} & Safety & 0.711 & 0.814 & 6.735 & 67.830 & 0.566 & 0.227 & \cellcolor{LightCyan}\textbf{0.776} \\ \hline
\end{tabular}}\vspace{-0.1cm}
\caption{\textbf{Performance in ProcTHOR FleeNav.} We evaluate each method in the validation set with five different configurations of objective prioritization: uniform reward weight across all objectives and prioritizing a single objective 3 times as much as other objectives. The displayed sub-reward values are normalized for each objective following Table~\ref{tab:procthor-objectnav}.}\label{tab:procthor-fleenav}\vspace{-0.3cm}
\end{table*}

\subsection{Reward Weight Prediction}\label{subsec:exp-reward-weight-prediction}

In the previous section, we have shown that our policy can effectively adjust its behavior to reflect the reward weights during inference. In this section, we compare three reward weight prediction methods and show the results of \MethodName for the full pipeline. As mentioned in Section~\ref{subsec:reward-weight-prediction}, the users have three distinct options to describe their preferences to the agent: (1) demonstrating a trajectory, (2) labeling their preferences on trajectory comparisons, and (3) providing language instructions.
Table~\ref{tab:weight-prediction-results} shows the quantitative performance of the three weight prediction methods, each with its own advantage. 

\begin{table}[]
\resizebox{0.48\textwidth}{!}{
\begin{tabular}{llrcc}
\hline
\multicolumn{3}{c}{\textbf{Weight Prediction Methods}} & \multirow{2}{*}{\textbf{Sim} $\uparrow$} & \multirow{2}{*}{\textbf{GGI} } \\
Input & Model & $N$ &  &  \\

\hline
Human Demonstrations & - & 1 & 0.707 & 0.347 \\

\cmidrule(r){1-3}
\multirow{7}{*}{\begin{tabular}[c]{@{}c@{}}Preference Feedback\end{tabular}} & \multirow{3}{*}{\begin{tabular}[c]{@{}l@{}}Pairwise\\ Comparison\\ (M=1)\end{tabular}} & 20 & 0.356 & 0.800 \\
 &  & 50 & 0.358 & 0.800 \\
 &  & 500 & 0.897 & 0.800 \\

\cmidrule(r){2-3}
&  
Group & 5 & 0.689 & 0.626 \\
 &  Comparison & 10 & 0.793 & 0.618\\
 &  (M=2) & 25 & \textbf{0.935} & 0.657 \\

 \cmidrule(r){2-3}
 & \multirow{3}{*}{\begin{tabular}[c]{@{}l@{}}Group\\ Comparison\\(M=5)\end{tabular}} & 2 & 0.722 & 0.634 \\
 &  & 4 & 0.682 & 0.762 \\
 &  & 10 & 0.862 & 0.641 \\
 
\cmidrule(r){1-3}
\multirow{4}{*}{\begin{tabular}[c]{@{}c@{}}Language Instructions\end{tabular}} &  ChatGPT & 1 & 0.530 & 0.388 \\
 &  \ \ w/ ICL & 1 & 0.529 & 0.379 \\
 & \ \ w/ CoT & 1 & 0.614 & 0.391 \\
 &  \ \ w/ ICL + CoT & 1 & 0.482 & 0.347 \\
\hline
\end{tabular}}
\caption{\textbf{Comparison of Three Weight Prediction Methods in ProcTHOR ObjectNav.} We predict the optimal reward weights from human demonstrations, preference feedback on trajectory comparisons, and language instructions. We measure the cosine similarity (Sim) between the predicted weights and the weights designed by human experts. We also calculate generalized gini index (GGI) which measures the peakedness of the predicted weights.}\label{tab:weight-prediction-results}\vspace{-0.3cm}
\end{table}

\noindent\textbf{Weight Prediction Performance.}
Weight optimization from human demonstrations shows $70.7\%$ cosine similarity between the predicted weights and the weights designed by human experts only using a single human demonstration. Preference feedback on group trajectory comparisons shows the highest prediction performance, $93.5\%$, when each group contains two trajectories.
Utilizing ChatGPT with four different settings based on the use of ICL and CoT, using ChatGPT with CoT resulted in the best performance.

\noindent\textbf{Peakedness of Weights.} In Table~\ref{tab:weight-prediction-results}, preference feedback on pairwise comparison shows the most peaked predicted weights while using human demonstrations outputs the least peaked weights. This could be due to the ambiguity lying in human demonstrations, where multiple similar reward weight vectors can produce the same trajectory. Although there appears to be no direct correlation between peakedness and weight prediction performance, this analysis provides valuable insights into the distinct characteristics of the different weight prediction methods.

\begin{table}[t!]
\resizebox{0.48\textwidth}{!}{
\begin{tabular}{llrc}
\hline
\multicolumn{3}{c}{\textbf{Weight Prediction Methods}} & \multirow{2}{*}{\textbf{Win Rate} $\uparrow$} \\
Input & Model & $N$ &  \\
\hline
Human Demo. & - & 1 & 0.556 \\
\cmidrule(r){1-3}
\multirow{4}{*}{\begin{tabular}[c]{@{}c@{}}Preference Feedback\end{tabular}} & Pairwise Comparison (M=1) & 50 & 0.552 \\
 \cmidrule(r){2-3}
 & Group Comparison (M=2) & 25 & \textbf{0.650} \\
 \cmidrule(r){2-3}
 & Group Comparison (M=5) & 10 & 0.588 \\
\cmidrule(r){1-3}
Language Instruction &  ChatGPT w/ CoT & 1 & 0.600 \\
\hline
\end{tabular}}
\caption{\textbf{Human Evaluation on Scenario-Trajectory Matching.} Participants evaluate trajectories generated with the trained policy and the reward weights predicted for five scenarios in ObjectNav.}\label{tab:human-eval}\vspace{-0.3cm}
\end{table}

\noindent\textbf{Human Evaluation.}
We also perform human evaluations by asking participants to compare trajectories generated with the predicted reward weights for different scenarios. Results in Table~\ref{tab:human-eval} show that group trajectory comparison, especially with two trajectories per group, achieves the highest win rate, significantly outperforming other methods by up to $17.8\%$. This high win rate indicates that the generated trajectories closely align with the intended scenarios. Among the weight prediction methods using preference feedback, group comparison with a group size of two requires only half the binary feedback compared to pairwise comparison, yet it improves the win rate significantly. More efficiently, group comparison with five trajectories per group needs just 10 user feedback while still managing a $6.5\%$ higher win rate than pairwise comparison. Interestingly, using language instructions to infer reward weights shows the second-best performance among all methods in Table~\ref{tab:human-eval}, demonstrating the potential of LLMs in understanding and translating complex human preferences into reward weights using world knowledge.

\vspace{0.1cm}
\mypara{Ablation Study.} Ablation studies on codebook and group trajectory comparison are provided in the supplementary. \vspace{-0.1cm}

\section{Conclusion}
This paper proposes \textit{Promptable Behaviors}, a novel framework that advances the personalization of robotic behaviors in complex environments, efficiently adapting to diverse human preferences with minimal user interaction. By leveraging MORL and three weight prediction methods, we have demonstrated the ability to prompt agent behaviors through reward weight adjustments in object-goal and flee navigation. For future work, we will demonstrate our method in various tasks such as manipulation. Additionally, since we assume static and linear preferences, we will extend our approach to consider dynamic and non-linear preferences.

{
    \small
    \bibliographystyle{ieeenat_fullname}
    \bibliography{references}

\begin{thebibliography}{71}
\providecommand{\natexlab}[1]{#1}
\providecommand{\url}[1]{\texttt{#1}}
\expandafter\ifx\csname urlstyle\endcsname\relax
  \providecommand{\doi}[1]{doi: #1}\else
  \providecommand{\doi}{doi: \begingroup \urlstyle{rm}\Url}\fi

\bibitem[Adams et~al.(2022)Adams, Cody, and Beling]{adams2022inverserl-survey}
Stephen Adams, Tyler Cody, and Peter~A Beling.
\newblock A survey of inverse reinforcement learning.
\newblock \emph{Artificial Intelligence Review}, 55\penalty0 (6):\penalty0 4307--4346, 2022.

\bibitem[Akrour et~al.(2011)Akrour, Schoenauer, and Sebag]{akrour2011preference}
Riad Akrour, Marc Schoenauer, and Michele Sebag.
\newblock Preference-based policy learning.
\newblock In \emph{Machine Learning and Knowledge Discovery in Databases: European Conference, ECML PKDD 2011, Athens, Greece, September 5-9, 2011. Proceedings, Part I 11}, pages 12--27. Springer, 2011.

\bibitem[Akrour et~al.(2012)Akrour, Schoenauer, and Sebag]{akrour2012april}
Riad Akrour, Marc Schoenauer, and Mich{\`e}le Sebag.
\newblock April: Active preference learning-based reinforcement learning.
\newblock In \emph{Machine Learning and Knowledge Discovery in Databases: European Conference, ECML PKDD 2012, Bristol, UK, September 24-28, 2012. Proceedings, Part II 23}, pages 116--131. Springer, 2012.

\bibitem[Akrour et~al.(2014)Akrour, Schoenauer, Sebag, and Souplet]{akrour2014programming}
Riad Akrour, Marc Schoenauer, Mich{\`e}le Sebag, and Jean-Christophe Souplet.
\newblock Programming by feedback.
\newblock In \emph{International Conference on Machine Learning}, number~32, pages 1503--1511. JMLR. org, 2014.

\bibitem[Alegre et~al.(2022)Alegre, Felten, Talbi, Danoy, Now{\'e}, Bazzan, and da~Silva]{Alegre+2022mo-gym}
Lucas~N. Alegre, Florian Felten, El-Ghazali Talbi, Gr{\'e}goire Danoy, Ann Now{\'e}, Ana L.~C. Bazzan, and Bruno~C. da Silva.
\newblock {MO-Gym}: A library of multi-objective reinforcement learning environments.
\newblock In \emph{Proceedings of the 34th Benelux Conference on Artificial Intelligence BNAIC/Benelearn 2022}, 2022.

\bibitem[Alegre et~al.(2023)Alegre, Bazzan, Roijers, Now{\'e}, and da~Silva]{alegre2023gpi-ls-pd}
Lucas~N Alegre, Ana~LC Bazzan, Diederik~M Roijers, Ann Now{\'e}, and Bruno~C da Silva.
\newblock Sample-efficient multi-objective learning via generalized policy improvement prioritization.
\newblock \emph{arXiv preprint arXiv:2301.07784}, 2023.

\bibitem[Arora and Doshi(2021)]{arora2021inverserl-survey}
Saurabh Arora and Prashant Doshi.
\newblock A survey of inverse reinforcement learning: Challenges, methods and progress.
\newblock \emph{Artificial Intelligence}, 297:\penalty0 103500, 2021.

\bibitem[Biyik and Sadigh(2018)]{biyik2018batch}
Erdem Biyik and Dorsa Sadigh.
\newblock Batch active preference-based learning of reward functions.
\newblock In \emph{Conference on robot learning (CoRL)}. PMLR, 2018.

\bibitem[B{\i}y{\i}k et~al.(2019)B{\i}y{\i}k, Lazar, Sadigh, and Pedarsani]{biyik2019green}
Erdem B{\i}y{\i}k, Daniel~A Lazar, Dorsa Sadigh, and Ramtin Pedarsani.
\newblock The green choice: Learning and influencing human decisions on shared roads.
\newblock In \emph{2019 IEEE 58th conference on decision and control (CDC)}. IEEE, 2019.

\bibitem[Bradley and Terry(1952)]{bradley1952rank}
Ralph~Allan Bradley and Milton~E Terry.
\newblock Rank analysis of incomplete block designs: I. the method of paired comparisons.
\newblock \emph{Biometrika}, 39\penalty0 (3/4):\penalty0 324--345, 1952.

\bibitem[Brown et~al.(2020)Brown, Mann, Ryder, Subbiah, Kaplan, Dhariwal, Neelakantan, Shyam, Sastry, Askell, et~al.]{brown2020incontext}
Tom Brown, Benjamin Mann, Nick Ryder, Melanie Subbiah, Jared~D Kaplan, Prafulla Dhariwal, Arvind Neelakantan, Pranav Shyam, Girish Sastry, Amanda Askell, et~al.
\newblock Language models are few-shot learners.
\newblock \emph{Advances in neural information processing systems}, 33:\penalty0 1877--1901, 2020.

\bibitem[Busa-Fekete et~al.(2017)Busa-Fekete, Sz{\"o}r{\'e}nyi, Weng, and Mannor]{busa2017multiobjective-gini}
R{\'o}bert Busa-Fekete, Bal{\'a}zs Sz{\"o}r{\'e}nyi, Paul Weng, and Shie Mannor.
\newblock Multi-objective bandits: Optimizing the generalized gini index.
\newblock In \emph{International Conference on Machine Learning}, pages 625--634. PMLR, 2017.

\bibitem[Cai et~al.(2023)Cai, Zhang, Zhao, Bian, Sugiyama, and Llorens]{cai2023distributional}
Xin-Qiang Cai, Pushi Zhang, Li Zhao, Jiang Bian, Masashi Sugiyama, and Ashley~Juan Llorens.
\newblock Distributional pareto-optimal multi-objective reinforcement learning.
\newblock In \emph{Thirty-seventh Conference on Neural Information Processing Systems}, 2023.

\bibitem[Cheng et~al.(2023)Cheng, Wang, Dong, Cai, and Sun]{cheng2023multi-crowd-aware}
Guangran Cheng, Yuanda Wang, Lu Dong, Wenzhe Cai, and Changyin Sun.
\newblock Multi-objective deep reinforcement learning for crowd-aware robot navigation with dynamic human preference.
\newblock \emph{Neural Computing and Applications}, pages 1--19, 2023.

\bibitem[Christiano et~al.(2017)Christiano, Leike, Brown, Martic, Legg, and Amodei]{christiano2017deep}
Paul~F Christiano, Jan Leike, Tom Brown, Miljan Martic, Shane Legg, and Dario Amodei.
\newblock Deep reinforcement learning from human preferences.
\newblock \emph{Advances in neural information processing systems (NeurIPS)}, 2017.

\bibitem[Daniel et~al.(2015)Daniel, Kroemer, Viering, Metz, and Peters]{daniel2015active}
Christian Daniel, Oliver Kroemer, Malte Viering, Jan Metz, and Jan Peters.
\newblock Active reward learning with a novel acquisition function.
\newblock \emph{Autonomous Robots}, 39:\penalty0 389--405, 2015.

\bibitem[Deitke et~al.(2020)Deitke, Han, Herrasti, Kembhavi, Kolve, Mottaghi, Salvador, Schwenk, VanderBilt, Wallingford, et~al.]{deitke2020robothor}
Matt Deitke, Winson Han, Alvaro Herrasti, Aniruddha Kembhavi, Eric Kolve, Roozbeh Mottaghi, Jordi Salvador, Dustin Schwenk, Eli VanderBilt, Matthew Wallingford, et~al.
\newblock Robothor: An open simulation-to-real embodied ai platform.
\newblock In \emph{Proceedings of the IEEE/CVF conference on computer vision and pattern recognition}, 2020.

\bibitem[Deitke et~al.(2022)Deitke, VanderBilt, Herrasti, Weihs, Ehsani, Salvador, Han, Kolve, Kembhavi, and Mottaghi]{deitke2022️procthor}
Matt Deitke, Eli VanderBilt, Alvaro Herrasti, Luca Weihs, Kiana Ehsani, Jordi Salvador, Winson Han, Eric Kolve, Aniruddha Kembhavi, and Roozbeh Mottaghi.
\newblock Procthor: Large-scale embodied ai using procedural generation.
\newblock \emph{Advances in Neural Information Processing Systems}, 35, 2022.

\bibitem[Deshpande et~al.(2021)Deshpande, Vaufreydaz, and Spalanzani]{deshpande2021navigation-ped-aware}
Niranjan Deshpande, Dominique Vaufreydaz, and Anne Spalanzani.
\newblock Navigation in urban environments amongst pedestrians using multi-objective deep reinforcement learning.
\newblock In \emph{2021 IEEE International Intelligent Transportation Systems Conference (ITSC)}, pages 923--928. IEEE, 2021.

\bibitem[Eftekhar et~al.(2023)Eftekhar, Zeng, Duan, Farhadi, Kembhavi, and Krishna]{eftekhar2023selective}
Ainaz Eftekhar, Kuo-Hao Zeng, Jiafei Duan, Ali Farhadi, Ani Kembhavi, and Ranjay Krishna.
\newblock Selective visual representations improve convergence and generalization for embodied ai.
\newblock \emph{arXiv preprint arXiv:2311.04193}, 2023.

\bibitem[El~Asri et~al.(2016)El~Asri, Piot, Geist, Laroche, and Pietquin]{el2016score}
Layla El~Asri, Bilal Piot, Matthieu Geist, Romain Laroche, and Olivier Pietquin.
\newblock Score-based inverse reinforcement learning.
\newblock 2016.

\bibitem[Fang et~al.(2021)Fang, Zhang, and Wang]{fang2021visual-inverserl}
Qiang Fang, Wenzhuo Zhang, and Xitong Wang.
\newblock Visual navigation using inverse reinforcement learning and an extreme learning machine.
\newblock \emph{Electronics}, 10\penalty0 (16):\penalty0 1997, 2021.

\bibitem[F{\"u}rnkranz et~al.(2012)F{\"u}rnkranz, H{\"u}llermeier, Cheng, and Park]{furnkranz2012preference}
Johannes F{\"u}rnkranz, Eyke H{\"u}llermeier, Weiwei Cheng, and Sang-Hyeun Park.
\newblock Preference-based reinforcement learning: a formal framework and a policy iteration algorithm.
\newblock \emph{Machine learning}, 89:\penalty0 123--156, 2012.

\bibitem[Hayes et~al.(2022)Hayes, R{\u{a}}dulescu, Bargiacchi, K{\"a}llstr{\"o}m, Macfarlane, Reymond, Verstraeten, Zintgraf, Dazeley, Heintz, et~al.]{hayes2022morl-survey}
Conor~F Hayes, Roxana R{\u{a}}dulescu, Eugenio Bargiacchi, Johan K{\"a}llstr{\"o}m, Matthew Macfarlane, Mathieu Reymond, Timothy Verstraeten, Luisa~M Zintgraf, Richard Dazeley, Fredrik Heintz, et~al.
\newblock A practical guide to multi-objective reinforcement learning and planning.
\newblock \emph{Autonomous Agents and Multi-Agent Systems}, 36\penalty0 (1):\penalty0 26, 2022.

\bibitem[Hejna and Sadigh(2023)]{hejna2023inverse}
Joey Hejna and Dorsa Sadigh.
\newblock Inverse preference learning: Preference-based rl without a reward function.
\newblock \emph{arXiv preprint arXiv:2305.15363}, 2023.

\bibitem[Hejna et~al.(2023)Hejna, Rafailov, Sikchi, Finn, Niekum, Knox, and Sadigh]{hejna2023contrastive}
Joey Hejna, Rafael Rafailov, Harshit Sikchi, Chelsea Finn, Scott Niekum, W~Bradley Knox, and Dorsa Sadigh.
\newblock Contrastive prefence learning: Learning from human feedback without rl.
\newblock \emph{arXiv preprint arXiv:2310.13639}, 2023.

\bibitem[Hejna~III and Sadigh(2023)]{hejna2023few}
Donald~Joseph Hejna~III and Dorsa Sadigh.
\newblock Few-shot preference learning for human-in-the-loop rl.
\newblock In \emph{Conference on Robot Learning}, pages 2014--2025. PMLR, 2023.

\bibitem[Hellou et~al.(2021)Hellou, Gasteiger, Lim, Jang, and Ahn]{hellou2021personalization}
Mehdi Hellou, Norina Gasteiger, Jong~Yoon Lim, Minsu Jang, and Ho~Seok Ahn.
\newblock Personalization and localization in human-robot interaction: A review of technical methods.
\newblock \emph{Robotics}, 10\penalty0 (4):\penalty0 120, 2021.

\bibitem[Hwang et~al.(2023{\natexlab{a}})Hwang, Jeong, Kim, Oh, and Oh]{hwang2023meta}
Minyoung Hwang, Jaeyeon Jeong, Minsoo Kim, Yoonseon Oh, and Songhwai Oh.
\newblock Meta-explore: Exploratory hierarchical vision-and-language navigation using scene object spectrum grounding.
\newblock In \emph{Proceedings of the IEEE/CVF Conference on Computer Vision and Pattern Recognition (CVPR)}, 2023{\natexlab{a}}.

\bibitem[Hwang et~al.(2023{\natexlab{b}})Hwang, Lee, Kee, Kim, Lee, and Oh]{hwang2023rlhf}
Minyoung Hwang, Gunmin Lee, Hogun Kee, Chan~Woo Kim, Kyungjae Lee, and Songhwai Oh.
\newblock Sequential preference ranking for efficient reinforcement learning from human feedback.
\newblock In \emph{Advances in Neural Information Processing Systems (NeurIPS)}, 2023{\natexlab{b}}.

\bibitem[Jang et~al.(2023)Jang, Kim, Lin, Wang, Hessel, Zettlemoyer, Hajishirzi, Choi, and Ammanabrolu]{jang2023personalized}
Joel Jang, Seungone Kim, Bill~Yuchen Lin, Yizhong Wang, Jack Hessel, Luke Zettlemoyer, Hannaneh Hajishirzi, Yejin Choi, and Prithviraj Ammanabrolu.
\newblock Personalized soups: Personalized large language model alignment via post-hoc parameter merging.
\newblock \emph{arXiv preprint arXiv:2310.11564}, 2023.

\bibitem[Khandelwal et~al.(2022)Khandelwal, Weihs, Mottaghi, and Kembhavi]{khandelwal2022simpleembclip}
Apoorv Khandelwal, Luca Weihs, Roozbeh Mottaghi, and Aniruddha Kembhavi.
\newblock Simple but effective: Clip embeddings for embodied ai.
\newblock In \emph{Proceedings of the IEEE/CVF Conference on Computer Vision and Pattern Recognition}, pages 14829--14838, 2022.

\bibitem[Kojima et~al.(2022)Kojima, Gu, Reid, Matsuo, and Iwasawa]{kojima2022chainofthought}
Takeshi Kojima, Shixiang~Shane Gu, Machel Reid, Yutaka Matsuo, and Yusuke Iwasawa.
\newblock Large language models are zero-shot reasoners.
\newblock \emph{Advances in neural information processing systems}, 35:\penalty0 22199--22213, 2022.

\bibitem[Kolve et~al.(2017)Kolve, Mottaghi, Han, VanderBilt, Weihs, Herrasti, Deitke, Ehsani, Gordon, Zhu, et~al.]{kolve2017ai2thor}
Eric Kolve, Roozbeh Mottaghi, Winson Han, Eli VanderBilt, Luca Weihs, Alvaro Herrasti, Matt Deitke, Kiana Ehsani, Daniel Gordon, Yuke Zhu, et~al.
\newblock Ai2-thor: An interactive 3d environment for visual ai.
\newblock \emph{arXiv preprint arXiv:1712.05474}, 2017.

\bibitem[Lee et~al.(2021)Lee, Smith, and Abbeel]{lee2021pebble}
Kimin Lee, Laura Smith, and Pieter Abbeel.
\newblock Pebble: Feedback-efficient interactive reinforcement learning via relabeling experience and unsupervised pre-training.
\newblock \emph{Proceedings of the International Conference on Machine Learning (ICML)}, 2021.

\bibitem[Liang et~al.(2022)Liang, Shu, Lee, and Abbeel]{liang2022reward}
Xinran Liang, Katherine Shu, Kimin Lee, and Pieter Abbeel.
\newblock Reward uncertainty for exploration in preference-based reinforcement learning.
\newblock \emph{Proceedings of the International Conference on Learning Representations (ICLR)}, 2022.

\bibitem[Liu et~al.(2022)Liu, Bai, Du, and Yang]{liumetarewardnet}
Runze Liu, Fengshuo Bai, Yali Du, and Yaodong Yang.
\newblock Meta-reward-net: Implicitly differentiable reward learning for preference-based reinforcement learning.
\newblock In \emph{Advances in Neural Information Processing Systems (NeurIPS)}, 2022.

\bibitem[Lu et~al.(2022)Lu, Herman, and Yu]{lu2022multi-CAPQL}
Haoye Lu, Daniel Herman, and Yaoliang Yu.
\newblock Multi-objective reinforcement learning: Convexity, stationarity and pareto optimality.
\newblock In \emph{The Eleventh International Conference on Learning Representations}, 2022.

\bibitem[Mnih et~al.(2016)Mnih, Badia, Mirza, Graves, Lillicrap, Harley, Silver, and Kavukcuoglu]{mnih2016asynchronous}
Volodymyr Mnih, Adria~Puigdomenech Badia, Mehdi Mirza, Alex Graves, Timothy Lillicrap, Tim Harley, David Silver, and Koray Kavukcuoglu.
\newblock Asynchronous methods for deep reinforcement learning.
\newblock In \emph{International conference on machine learning}, pages 1928--1937. PMLR, 2016.

\bibitem[Mossalam et~al.(2016)Mossalam, Assael, Roijers, and Whiteson]{mossalam2016deepmorl}
Hossam Mossalam, Yannis~M Assael, Diederik~M Roijers, and Shimon Whiteson.
\newblock Multi-objective deep reinforcement learning.
\newblock \emph{arXiv preprint arXiv:1610.02707}, 2016.

\bibitem[Myers et~al.(2022)Myers, Biyik, Anari, and Sadigh]{myers2022learning}
Vivek Myers, Erdem Biyik, Nima Anari, and Dorsa Sadigh.
\newblock Learning multimodal rewards from rankings.
\newblock In \emph{Conference on Robot Learning}, pages 342--352. PMLR, 2022.

\bibitem[{OpenAI}(2022)]{chatgpt}
{OpenAI}.
\newblock {ChatGPT}.
\newblock \url{https://openai.com/blog/chatgpt}, 2022.

\bibitem[Ouyang et~al.(2022)Ouyang, Wu, Jiang, Almeida, Wainwright, Mishkin, Zhang, Agarwal, Slama, Ray, et~al.]{ouyang2022instructgpt}
Long Ouyang, Jeffrey Wu, Xu Jiang, Diogo Almeida, Carroll Wainwright, Pamela Mishkin, Chong Zhang, Sandhini Agarwal, Katarina Slama, Alex Ray, et~al.
\newblock Training language models to follow instructions with human feedback.
\newblock \emph{Advances in Neural Information Processing Systems}, 35:\penalty0 27730--27744, 2022.

\bibitem[Pan et~al.(2020)Pan, Xu, Wang, and Ren]{pan2020additional}
Anqi Pan, Wenjun Xu, Lei Wang, and Hongliang Ren.
\newblock Additional planning with multiple objectives for reinforcement learning.
\newblock \emph{Knowledge-Based Systems}, 193:\penalty0 105392, 2020.

\bibitem[Park et~al.(2022)Park, Seo, Shin, Lee, Abbeel, and Lee]{park2022surf}
Jongjin Park, Younggyo Seo, Jinwoo Shin, Honglak Lee, Pieter Abbeel, and Kimin Lee.
\newblock Surf: Semi-supervised reward learning with data augmentation for feedback-efficient preference-based reinforcement learning.
\newblock \emph{Proceedings of the International Conference on Learning Representations (ICLR)}, 2022.

\bibitem[Peschl et~al.(2021)Peschl, Zgonnikov, Oliehoek, and Siebert]{peschl2021moral}
Markus Peschl, Arkady Zgonnikov, Frans~A Oliehoek, and Luciano~C Siebert.
\newblock Moral: Aligning ai with human norms through multi-objective reinforced active learning.
\newblock \emph{arXiv preprint arXiv:2201.00012}, 2021.

\bibitem[Radford et~al.(2021)Radford, Kim, Hallacy, Ramesh, Goh, Agarwal, Sastry, Askell, Mishkin, Clark, et~al.]{radford2021clip}
Alec Radford, Jong~Wook Kim, Chris Hallacy, Aditya Ramesh, Gabriel Goh, Sandhini Agarwal, Girish Sastry, Amanda Askell, Pamela Mishkin, Jack Clark, et~al.
\newblock Learning transferable visual models from natural language supervision.
\newblock In \emph{International conference on machine learning}, pages 8748--8763. PMLR, 2021.

\bibitem[Ramrakhya et~al.(2022)Ramrakhya, Undersander, Batra, and Das]{ramrakhya2022habitat-web}
Ram Ramrakhya, Eric Undersander, Dhruv Batra, and Abhishek Das.
\newblock Habitat-web: Learning embodied object-search strategies from human demonstrations at scale.
\newblock In \emph{Proceedings of the IEEE/CVF Conference on Computer Vision and Pattern Recognition}, pages 5173--5183, 2022.

\bibitem[Ramrakhya et~al.(2023)Ramrakhya, Batra, Wijmans, and Das]{ramrakhya2023pirlnav}
Ram Ramrakhya, Dhruv Batra, Erik Wijmans, and Abhishek Das.
\newblock Pirlnav: Pretraining with imitation and rl finetuning for objectnav.
\newblock \emph{arXiv preprint arXiv:2301.07302}, 2023.

\bibitem[Ren et~al.(2023)Ren, Dixit, Bodrova, Singh, Tu, Brown, Xu, Takayama, Xia, Varley, et~al.]{ren2023askforhelp}
Allen~Z Ren, Anushri Dixit, Alexandra Bodrova, Sumeet Singh, Stephen Tu, Noah Brown, Peng Xu, Leila Takayama, Fei Xia, Jake Varley, et~al.
\newblock Robots that ask for help: Uncertainty alignment for large language model planners.
\newblock \emph{arXiv preprint arXiv:2307.01928}, 2023.

\bibitem[Reymond et~al.(2022)Reymond, Bargiacchi, and Now{\'e}]{reymond2022pareto-PCN}
Mathieu Reymond, Eugenio Bargiacchi, and Ann Now{\'e}.
\newblock Pareto conditioned networks.
\newblock \emph{arXiv preprint arXiv:2204.05036}, 2022.

\bibitem[Roijers(2016)]{roijers2016multi}
Diederik~M Roijers.
\newblock Multi-objective decision-theoretic planning.
\newblock \emph{AI Matters}, 2\penalty0 (4):\penalty0 11--12, 2016.

\bibitem[Roijers et~al.(2018)Roijers, Steckelmacher, and Now{\'e}]{roijers2018multi}
Diederik~M Roijers, Denis Steckelmacher, and Ann Now{\'e}.
\newblock Multi-objective reinforcement learning for the expected utility of the return.
\newblock In \emph{Proceedings of the Adaptive and Learning Agents workshop at FAIM}, 2018.

\bibitem[Sadigh et~al.(2017)Sadigh, Dragan, Sastry, and Seshia]{sadigh2017active}
Dorsa Sadigh, Anca~D Dragan, Shankar Sastry, and Sanjit~A Seshia.
\newblock Active preference-based learning of reward functions.
\newblock In \emph{Proceedings of Robotics: Science and Systems (RSS)}, 2017.

\bibitem[Schulman et~al.(2017)Schulman, Wolski, Dhariwal, Radford, and Klimov]{schulman2017proximal}
John Schulman, Filip Wolski, Prafulla Dhariwal, Alec Radford, and Oleg Klimov.
\newblock Proximal policy optimization algorithms.
\newblock \emph{arXiv preprint arXiv:1707.06347}, 2017.

\bibitem[Siddique et~al.(2020)Siddique, Weng, and Zimmer]{siddique2020learning}
Umer Siddique, Paul Weng, and Matthieu Zimmer.
\newblock Learning fair policies in multi-objective (deep) reinforcement learning with average and discounted rewards.
\newblock In \emph{International Conference on Machine Learning}, pages 8905--8915. PMLR, 2020.

\bibitem[Singh et~al.(2022{\natexlab{a}})Singh, Kumar, and Singh]{singh2022robot-learning-survey-rl}
Bharat Singh, Rajesh Kumar, and Vinay~Pratap Singh.
\newblock Reinforcement learning in robotic applications: a comprehensive survey.
\newblock \emph{Artificial Intelligence Review}, pages 1--46, 2022{\natexlab{a}}.

\bibitem[Singh et~al.(2022{\natexlab{b}})Singh, Weihs, Herrasti, Choi, Kembhavi, and Mottaghi]{singh2022ask4help}
Kunal~Pratap Singh, Luca Weihs, Alvaro Herrasti, Jonghyun Choi, Aniruddha Kembhavi, and Roozbeh Mottaghi.
\newblock Ask4help: Learning to leverage an expert for embodied tasks.
\newblock \emph{Advances in Neural Information Processing Systems}, 35:\penalty0 16221--16232, 2022{\natexlab{b}}.

\bibitem[Sugiyama et~al.(2012)Sugiyama, Meguro, and Minami]{sugiyama2012preference}
Hiroaki Sugiyama, Toyomi Meguro, and Yasuhiro Minami.
\newblock Preference-learning based inverse reinforcement learning for dialog control.
\newblock In \emph{Thirteenth Annual Conference of the International Speech Communication Association}, 2012.

\bibitem[Van~Moffaert and Now{\'e}(2014)]{van2014multi-ParetoQ}
Kristof Van~Moffaert and Ann Now{\'e}.
\newblock Multi-objective reinforcement learning using sets of pareto dominating policies.
\newblock \emph{The Journal of Machine Learning Research}, 15\penalty0 (1):\penalty0 3483--3512, 2014.

\bibitem[Van~Moffaert et~al.(2013)Van~Moffaert, Drugan, and Now{\'e}]{van2013scalarized-MOQL}
Kristof Van~Moffaert, Madalina~M Drugan, and Ann Now{\'e}.
\newblock Scalarized multi-objective reinforcement learning: Novel design techniques.
\newblock In \emph{2013 IEEE symposium on adaptive dynamic programming and reinforcement learning (ADPRL)}, pages 191--199. IEEE, 2013.

\bibitem[Weihs et~al.(2020)Weihs, Salvador, Kotar, Jain, Zeng, Mottaghi, and Kembhavi]{AllenAct}
Luca Weihs, Jordi Salvador, Klemen Kotar, Unnat Jain, Kuo-Hao Zeng, Roozbeh Mottaghi, and Aniruddha Kembhavi.
\newblock Allenact: A framework for embodied ai research.
\newblock \emph{arXiv preprint arXiv:2008.12760}, 2020.

\bibitem[Weymark(1981)]{weymark1981generalizedginiindex}
John~A Weymark.
\newblock Generalized gini inequality indices.
\newblock \emph{Mathematical Social Sciences}, 1\penalty0 (4):\penalty0 409--430, 1981.

\bibitem[Wijmans et~al.(2019)Wijmans, Kadian, Morcos, Lee, Essa, Parikh, Savva, and Batra]{wijmans2019ddppo}
Erik Wijmans, Abhishek Kadian, Ari Morcos, Stefan Lee, Irfan Essa, Devi Parikh, Manolis Savva, and Dhruv Batra.
\newblock Dd-ppo: Learning near-perfect pointgoal navigators from 2.5 billion frames.
\newblock In \emph{arXiv preprint arXiv:1911.00357}, 2019.

\bibitem[Wilson et~al.(2012)Wilson, Fern, and Tadepalli]{wilson2012bayesian}
Aaron Wilson, Alan Fern, and Prasad Tadepalli.
\newblock A bayesian approach for policy learning from trajectory preference queries.
\newblock \emph{Advances in neural information processing systems}, 25, 2012.

\bibitem[Wirth et~al.(2016)Wirth, F{\"u}rnkranz, and Neumann]{wirth2016model}
Christian Wirth, Johannes F{\"u}rnkranz, and Gerhard Neumann.
\newblock Model-free preference-based reinforcement learning.
\newblock In \emph{Proceedings of the AAAI Conference on Artificial Intelligence}, 2016.

\bibitem[Wirth et~al.(2017)Wirth, Akrour, Neumann, F{\"u}rnkranz, et~al.]{wirth2017pbrl-survey}
Christian Wirth, Riad Akrour, Gerhard Neumann, Johannes F{\"u}rnkranz, et~al.
\newblock A survey of preference-based reinforcement learning methods.
\newblock \emph{Journal of Machine Learning Research}, 18\penalty0 (136):\penalty0 1--46, 2017.

\bibitem[Xiao et~al.(2022)Xiao, Liu, Warnell, and Stone]{xiao2022robot-learning-survey-navigation}
Xuesu Xiao, Bo Liu, Garrett Warnell, and Peter Stone.
\newblock Motion planning and control for mobile robot navigation using machine learning: a survey.
\newblock \emph{Autonomous Robots}, 46\penalty0 (5):\penalty0 569--597, 2022.

\bibitem[Xu et~al.(2020)Xu, Tian, Ma, Rus, Sueda, and Matusik]{xu2020prediction-PGMORL}
Jie Xu, Yunsheng Tian, Pingchuan Ma, Daniela Rus, Shinjiro Sueda, and Wojciech Matusik.
\newblock Prediction-guided multi-objective reinforcement learning for continuous robot control.
\newblock In \emph{International conference on machine learning}, pages 10607--10616. PMLR, 2020.

\bibitem[Yang et~al.(2019)Yang, Sun, and Narasimhan]{yang2019generalized-envQ}
Runzhe Yang, Xingyuan Sun, and Karthik Narasimhan.
\newblock A generalized algorithm for multi-objective reinforcement learning and policy adaptation.
\newblock \emph{Advances in neural information processing systems}, 32, 2019.

\bibitem[Zhu et~al.(2021)Zhu, Liang, Zhu, Yu, Chang, and Liang]{zhu2021vln-soon}
Fengda Zhu, Xiwen Liang, Yi Zhu, Qizhi Yu, Xiaojun Chang, and Xiaodan Liang.
\newblock Soon: Scenario oriented object navigation with graph-based exploration.
\newblock In \emph{Proceedings of the IEEE/CVF Conference on Computer Vision and Pattern Recognition}, pages 12689--12699, 2021.

\end{thebibliography}
}

\clearpage
\appendix
\renewcommand\thesection{\Alph{section}}
\twocolumn[\vspace{1.1cm}\section*{\Large\centering Supplementary Material for ``Promptable Behaviors:\\
Personalizing Multi-Objective Rewards from Human Preferences"\vspace{0.6cm}}
\large
\lineskip .5em
\begin{center}
\textbf{Minyoung Hwang$^{1}$, Luca Weihs$^{1}$, Chanwoo Park$^{2}$, Kimin Lee$^{3}$,}\\ \textbf{Aniruddha Kembhavi$^{1}$, Kiana Ehsani$^{1}$}
\\
$^{1}$PRIOR @ Allen Institute for AI, $^{2}$Massachusetts Institute of Technology, \\
$^{3}$Korea Advanced Institute of Science and Technology
\end{center}]

We provide additional details and analyses of the proposed method in this supplementary material. Section~\ref{sec:group-trajectory-comparison} provides detailed algorithm and theoretical analyses for group trajectory comparison. Section~\ref{sec:experiment-setup} provides implementation details and experiment settings. Section~\ref{sec:experiment-results} provides evaluation results for RoboTHOR with detailed analyses on both ProcTHOR and RoboTHOR experiments. Section~\ref{sec:ablation-study} provides implementation details and detailed results for the ablation study.

\section{Group Trajectory Comparison}\label{sec:group-trajectory-comparison}
\subsection{Algorithm}
\label{ssec:group-trajectory-comparison-alg}
\RestyleAlgo{ruled}
\setlength{\textfloatsep}{5pt}
\begin{algorithm*}[h]
\SetAlgoLined
\caption{Group Trajectory Comparison with Probabilistic Rejective Sampling}\label{alg:group-trajectory-comparison}
\textbf{Require:} \# Queries: $N$, Rejection Probability: $1-\delta$\;
\# Objectives: $K$, \# Trajectories per Group: $M$\;
\textbf{Initialize:} Preference Buffer $\mathcal{D}\gets \emptyset$, Constraint Buffer $\mathcal{C}\gets \emptyset$\\
\For{$n=1$ \KwTo $N$}{
    // Sample two groups of weights $G^n_1, G^n_2$ from the set of all possible weights $\Delta_K$\\
    // Generate $M$ trajectories for each group using the weights in the group\\
    // Calculate sub-rewards for $2\cdot M$ trajectories across $K$ objectives\\
    $\mathbf{R}^n_i \!\gets\!\! [\mathbf{r}(\tau^i_1),\! ...,\! \mathbf{r}(\tau^i_{M})]$ \! s.t. \!\! $\mathbf{\tau}^i_j\!\sim\! \pi(\cdot|\fw), \fw\in G^n_i$\;
    // Ask the user to give a preference label for these two groups\\
    $y_n \gets$ preference label for pair $(G^n_1, G^n_2)$\\
    // Determine an inequality constraint based on the user's preference\\
    Choose $\mathbf{a}_n\in\mathbb{R}^K, \fc_n\in \RR^K, b_n\in\mathbb{R}$ s.t. 
            $\mathbf{R}^n_2-\mathbf{R}^n_1 = \mathbf{a}_n{\fc_n}^\intercal - (b_n+0.5)\mathbf{1}{\fc_n}^\intercal$\\
    \eIf{User prefers the first group}{
        Ensure the probability of this preference being correct is above the threshold $\delta$\\
    }{
        Ensure the probability of the opposite preference being correct is above the threshold $\delta$\\
    }
    
    // Update Preference Buffer with the user's choice\\
    $\mathcal{D} \gets \mathcal{D} \cup ((\sigma_n, \sigma_{n+1}), y_n)$\\
    // Update Constraint Buffer with the new constraint\\
    $\mathcal{C} \gets$ $\mathcal{C} \cup \{\mathbf{a}_n^\intercal \mathbf{w}<b_2$\}\\
}
\textbf{Constrained Optimization:}\\
// Maximize the expected likelihood of the user's preferences being 
correct, considering all constraints\\
$\max\limits_{\mathbf{w}\in\Delta_K}\mathbb{E}[\log P(y_n|G^n_1,G^n_2;\mathbf{w})]$ s.t. $\mathbf{w}\in \cap_{n=1}^N C_i$
\end{algorithm*}

Algorithm~\ref{alg:group-trajectory-comparison} summarizes the reward weight prediction process from preference feedback on group trajectory comparisons. 
Suppose we are analyzing human preferences between two groups of weights in a simplex \( \Delta_K \). The groups are defined as $G_1 = \{\mathbf{w} \mid \mathbf{a}^\intercal \mathbf{w} > b+1\} \subseteq \Delta_K$ and $G_2 = \{\mathbf{w} \mid \mathbf{a}^\intercal \mathbf{w} < b\} \subseteq \Delta_K$, where \( \mathbf{a} \in \mathbb{R}^K \) and \( b \in \mathbb{R} \). Suppose we collect two groups of trajectories corresponding to $G_1$ and $G_2$ as $\mathcal{T}_i=\{\{\tau_{i,j}\}_{j =1}^{M} \mid\tau_{i,j}\sim \pi(\cdot|\fw) \text{ s.t. } \fw\sim \text{Unif}(G_i)\}$ for $\forall i\in\{1,2\}$ where $M = |\cT_1| = |\cT_2|$. We define group preference that $G_1$ is preferred to $G_2$ if $|\{\tau_{1,j}\succ\tau_{2,j}\}|>\alpha M$, where $M$ is the size of $\mathcal{T}_1$ and $\mathcal{T}_2$, $\tau_{i,j}\in\mathcal{T}_i$ is a trajectory generated with a reward weight vector from $G_i$ in the $j^{th}$ episode, and $\alpha\geq1/2$ is a threshold that determines the group preference. As described in Algorithm~\ref{alg:group-trajectory-comparison}, we eliminate the volume corresponding to \( G_2 \) if the human prefers \( G_1 \) over \( G_2 \). Conversely, if \( G_2 \) is preferred over \( G_1 \), the volume in \( G_1 \) is removed. This process is repeated for multiple iterations and we perform constrained optimization to find the reward weight vector from the left weight space that maximizes the likelihood of group preference.

\subsection{Theoretical Analyses}
For theoretical analysis, we assume the well-constructed set of trajectory pairs $\{\tau_{1, i}, \tau_{2, i}\}_{i=1}^{M}$ with a specific reward vector so that $\tau_{1,i}$ can be regarded as a trajectory sampled from $\fw$ in $G_1$, and $\tau_{2,i}$ can be regarded as a trajectory sampled from $\fw$ in $G_2$. This is not exactly the same as the sampling we did in \Cref{alg:group-trajectory-comparison}, where we uniformly sample $\fw$ from each group $G_i$. The justification for this assumption will be thoroughly detailed in this section.

\paragraph{Construction of trajectory pair set. }
We will \textit{construct and present pairs of trajectories}, denoted as \( \{(\tau_{1,i}, \tau_{2,i})\}_{i=1}^M \), where \( M \) represents the number of pairs in the group comparisons. For each trajectory \( \tau_{j,i} \), it is associated with a reward, denoted as \( r_{j,i,k} \). Here, \( k \in [K] \) specifies the \( k \)th objective, \( j = 1,2 \) indicates the first or second component of each pair, and \( i \in [M] \) corresponds to the \( i \)th pair in the group comparison data. Then, if a human prefer $\tau_{i_p, i}$ to $\tau_{3-i_p, i}$ for every $i \in [M]$, (i.e., $i_p$ is the index of preferred group for $i$th pair, which is 1 or 2), 
\begin{align}
    P&(\{\tau_{i_p,i}\succ \tau_{3-i_p, i}\}_{i=1}^M) \nonumber
    \\
    &= \prod_{i=1}^M \frac{\exp(\mathbf{\mathbf{w}^*}^\intercal \mathbf{r}(\tau_{i_p, i}))}{\exp(\mathbf{\mathbf{w}^*}^\intercal \mathbf{r}(\tau_{i_p, i})) + \exp(\mathbf{\mathbf{w}^*}^\intercal \mathbf{r}(\tau_{3- i_p, i}))} \label{eqn:gp-comparison-1}
\end{align}
holds due to the Bradley-Terry Model. Since we can construct $\tau_i$ (as we have a lot of trajectories), we that $r_{2,i,k} - r_{1,i,k} = c_i a_k - c_i \frac{b_1 + b_2}{2}$ for all $i \in [M], k \in [K]$. Without loss of generality, we set $b_1-b_2=0.5$.
We also set $c_i$ as a constant ${1.4}/(b_1-b_2)$. 
Then, \Cref{eqn:gp-comparison-1} can be written as 
\begin{align}
    P&(\{\tau_{i_p,i}\succ \tau_{3-i_p, i}\}_{i=1}^M) \nonumber
    \\
    &= \prod_{i=1}^M \frac{1}{1 + \exp((-1)^{\pmb{1}(i_p= 1)}c_i(\mathbf{\mathbf{w}^*}^\intercal \mathbf{a} - \frac{b_1 + b_2}{2} )}. \label{eqn:gp-comparison-2}
\end{align}
Here, \Cref{eqn:gp-comparison-2} holds due to dividing \Cref{eqn:gp-comparison-1}'s denominator and numerator by $\exp(\mathbf{\mathbf{w}^*}^\intercal \mathbf{r}(\tau_{i_p, i}))$. 

\paragraph{The justification for the assumption. }
A careful construction of $r_{1, i}$ and $r_{2,i}$ indicates that if $w^\star \in G_1$, i.e. $w^\star a > b_1$, $P(\tau_{1, i} \succ \tau_{2, i}) = \frac{1}{1 + \exp( -c_i(\mathbf{\mathbf{w}^*}^\intercal \mathbf{a} - \frac{b_1 + b_2}{2} )} >1/2$, which says that the possibility of preferring $\tau_{1,i}$ to $\tau_{2,i}$ is larger than preferring $\tau_{2,i}$ to $\tau_{1,i}$. In the same way, if $w^\star \in G_2$, then they will probably more prefer $\tau_{2,i}$ to $\tau_{1,i}$. Therefore, we can regard preferring $\tau_{j,i}$ to $\tau_{3-j, i}$ as preferring $j$ to $3-j$. In this line, we assume that group $g$ will be said as preferred if $|\{\tau_{g,i}\succ \tau_{3-g, i}\}_{i=1}^M| \geq \alpha M$ where $\alpha > 1/2$. 

\noindent \textbf{Probabilistic Guarantee.} Theorem~\ref{thm:group-preference} shows that with a sufficiently large group size, group comparison can achieve a probabilistic guarantee of predicting group preference with an error less than $\delta$.

\begin{theorem}\label{thm:group-preference}
If $M \geq \frac{-\log \delta - 1/2}{\alpha\left(1 +0.25 c\right) - 1- (1-\alpha)\log(1-\alpha)}$ holds, ${\fontdimen2\font=1.5ptP(G_2 \succ G_1 \mid \mathbf{w}^* \in G_1)\leq\delta}$ with probability at least $1-\delta$.
\end{theorem}

\begin{proof}
We will use the trajectory pairs defined in the aforementioned section. For each comparison between groups, we can decide the values of \( a \) and \( b_1, b_2 \), thereby shrinking the domain of possible \( \mathbf{w} \) values in \( \Delta_K \) using a specific affine hyperplane. Iteratively employing these group comparisons will provide a high confidence level in determining the real value of \( \mathbf{w}^* \).

\noindent Assume $\mathbf{w}^\star \in G_1$, but the case that human will select $G_2$ is
\begingroup
\allowdisplaybreaks
\begin{align}
P&(G_2 \succ G_1 \mid w^\star \in G_1)= 
    P_{\mathbf{w}^*}(|\{\tau_{2,i}\succ \tau_{1, i}\}_{i=1}^M| \geq \alpha M) \nonumber
    \\
    &= \sum_{k \geq \alpha M}  \begin{pmatrix}
M \\
k 
\end{pmatrix} \left(\frac{1}{1 + \exp(c_i(\mathbf{\mathbf{w}^*}^\intercal \mathbf{a} - \frac{b_1 + b_2}{2} )}\right)^k\nonumber
\\
&\qquad \qquad \qquad \left(\frac{1}{1 + \exp(-c_i(\mathbf{\mathbf{w}^*}^\intercal \mathbf{a} - \frac{b_1 + b_2}{2} )}\right)^{M-k} \nonumber
\\
&\leq \sum_{k \geq \alpha M}  \begin{pmatrix}
M \\
k 
\end{pmatrix} \left(\frac{1}{1 + \exp(c_i\frac{b_1 - b_2}{2})}\right)^{\alpha M}\nonumber
\\
&\qquad \qquad \qquad \left(\frac{1}{1 + \exp(c_i\frac{b_2 - b_1}{2})}\right)^{M-\alpha M} \label{eqn:gp-comparison-3}
\\
&\leq \frac{M^{M - \alpha M+1}}{e \sqrt{M - \alpha M} \left(\frac{M - \alpha M}{e}\right)^{M-\alpha M}}\left(\frac{1}{1 + \exp(c\frac{b_1 - b_2}{2})}\right)^{\alpha M} \nonumber
\\
&\qquad \qquad \qquad \left(\frac{1}{1 + \exp(c\frac{b_2 - b_1}{2})}\right)^{M-\alpha M} \label{eqn:gp-comparison-4}
\\
&\leq \exp\left(M - \alpha M - \frac{b_1 - b_2}{2}c \alpha M + \frac{1}{2}\right)/(1-\alpha)^{M-\alpha M} \nonumber
\\
&\leq \delta\nonumber
\end{align} 
\endgroup
if $M \geq \frac{-\log \delta - 1/2}{\alpha\left(1 +\frac{b_1-b_2}{2}c\right) - 1- (1-\alpha)\log(1-\alpha)}$ holds, so we can get exponentially small error rate with respect to $M$. Here, \Cref{eqn:gp-comparison-3} used that $b_1 - b_2 >0$, \Cref{eqn:gp-comparison-4} used $\begin{pmatrix}
M \\
k 
\end{pmatrix} \leq \begin{pmatrix}
M \\
\alpha M 
\end{pmatrix}$ and using Stirling formula of $\begin{pmatrix}
M \\
\alpha M 
\end{pmatrix}$. 

\end{proof}

\noindent Since $c_i>\frac{2-\log(2)}{b_1-b_2}$ and 
$\frac{1 + (1-\alpha)\log(1-\alpha)}{\alpha}<2-\log(2)$ for any $\alpha>\frac{1}{2}$,
$\alpha\left(1 +\frac{b_1-b_2}{2}c\right) - 1- (1-\alpha)\log(1-\alpha)>0$. 
Suppose the threshold of group preference $\alpha=2/3$ and we want an accuracy of $\delta=0.05$. Then, $P(G_2 \succ G_1 \mid \mathbf{w}^* \in G_1)\leq0.05$ if $M>4.996$. This theoretical analysis implies that we can efficiently remove a volume of the weight space by rejecting highly unprobable weights. We can do the same analysis for  $P(G_1 \succ G_2 \mid w^* \in G_2)$. Therefore, even if we have a small $M$, we can efficiently make the error smaller, which means that group comparison accurately removes the volume. 

\section{Experiment Setup}\label{sec:experiment-setup}
\subsection{Implementation Details}\label{subsec:implementation-details}
The agent observes an RGB image observation with a field of view of $63.453^\circ$ at each timestep. 
We define the time efficiency reward as a negative constant value at each timestep, fixed as $-0.01$ for all experiments. The path efficiency reward is defined as $\max(d_{t-1}, d_t)$, where $d_t$ is the distance from the agent's current location to the target location and $d_{t-1}$ is the distance from the agent to the target location at the previous timestep.
House exploration reward considers the navigation history, by encouraging the agent to visit new areas with its corresponding reward as a constant $r_{house\_explore}$ if the agent is visiting its current location for the first time and the target object is not observed yet. If the target object is observed or the agent revisits a location, the house exploration reward is $0$. $r_{house\_explore}$ is $0.1$ for all experiments. The object exploration reward is calculated by counting the number of observed objects, where the reward is defined as $r_{object\_found}*n_{new\_visible} / n_{total}$, where $r_{object\_found}$ is a constant $4.0$, $n_{new\_visible}$ is the number of objects newly observed at the current timestep, and $n_{total}$ is the number of all objects in the environment. Safety reward is defined as $-r_{safety}\cdot n_{unreachable}$ if the target object is visible in current observation and $0$ otherwise, where $r_{safety}$ is $0.005$. The number of unreachable locations near the agent $n_{unreachable}$ is higher than a threshold $n_{safety\_threshold}$. We draw a square grid of size $13\times 13$ with the agent's current location in the center of the grid and a distance $0.25 m$ unit grid size.

\subsection{List of Scenarios}
\begin{itemize}
    \item \textbf{Urgent:} The user is getting late to an important meeting and needs to quickly find an object in the house.
    \item \textbf{Energy Conservation:} The user wants to check an appliance in the house while the user is away, but the robot that has a limited battery life.
    \item \textbf{New Home:} The user just moved in and wants to find which furniture or object is located while inspecting the layout of the house as a video.
    \item \textbf{Post-Rearrangement:} After rearranging the house, the user does not remember where certain objects were placed. The user wants to find a specific object, while also inspecting other areas to confirm the new arrangement.
    \item \textbf{Quiet Operation:} At midnight, the user wants to find an object in the house without disturbing a sleeping child with any loud noise.

\end{itemize}

\subsection{Language Instruction Templates.}
The input template for the language model is as follows:

{\footnotesize
\begin{verbatim}
{Task Description}
Objectives are {Definitions}.
Given an instruction, I want to know the weights 
over the {K} objectives {Names of Objectives}.
The weights should be spiked, meaning that the 
weight of the most important objective should be 
much higher than the weight of the least important 
objective.

Here are some examples.
1. Instruction: I want the agent to explore  ...
   Rationale: The instruction describes that ...
   Answer: [0.05,0.1,0.65,0.1,0.1]
...

Instruction: {Instruction}
Rationale:
Answer: 
\end{verbatim}
}

\noindent For ObjectNav, we provide the following template:

{\footnotesize
\begin{spverbatim}
"""
In the object-goal navigation task in ProcTHOR, an agent is placed within a simulated environment containing various rooms and objects. The agent's main goal is to find a specific object in this environment. To assist the agent in its navigation, it can be given different objectives that determine how it behaves during its search. 

Objectives are:
1. Time Efficiency: Aim to find the target object using as few steps as possible.
2. Path Efficiency: Approach the goal using the most direct route. Consider if you're taking the shortest possible path.
3. House Exploration: Strive to explore the house thoroughly. This involves checking many different areas/rooms until you locate the target object.
4. Safety: Navigate while avoiding obstacles and areas where you could get trapped or stuck. 
5. Object Exploration: While finding the target object, try to inspect as many objects as you encounter.

Given a scenario, I want to know the weights over the five objectives (Time Efficiency, Path Efficiency, House Exploration, Safety, Object Exploration).
The weights should be spiked, meaning that the weight of the most important objective should be much higher than the weight of the least important objective.
The answer should be a list of five float numbers, summed to 1.

Here are some examples.
1. Scenario: My kid is asleep. Navigate to an apple in the kitchen without making any noise."
   Rationale: Based on the scenario, the agent should prioritize safety the most, assigning 0.6. Other objectives are not mentioned, assigning (1-0.6)/4=0.1 for each objective.
   Answer: [0.1,0.1,0.1,0.1,0.6]
2. Scenario: I am in hurry. I want to find an object before I am late for work.
    Rationale: Based on the scenario, time efficiency is the most important, assigning 0.6. Other objectives are not mentioned, assigning (1-0.6)/4=0.1 for each objective.
    Answer: [0.6,0.1,0.1,0.1,0.1]
3. Scenario: I want to find a missing object in my house. I looked into every room briefly but I couldn't find it.
    Rationale: Based on the scenario, the agent should explore the house thoroughly, assigning 0.4 for both house exploration and object exploration. Other objectives are not mentioned, assigning (1-0.4-0.4)/3=0.067 for each objective.
    Answer: [0.067,0.067,0.4,0.067,0.4]
4. Scenario: I bought an expensive furniture in my house. I want to find an object, but I don't want to damage the furniture.
    Rationale: Based on the scenario, the agent should prioritize safety the most, assigning 0.6. Other objectives are not mentioned, assigning (1-0.6)/4=0.1 for each objective.
    Answer: [0.1,0.1,0.1,0.1,0.6]
5. Scenario: I'm recording a video in the living room. While I'm working on this, I want the agent to find an object for me. I don't want the agent to move around too much since it might be too noisy and appear a lot in the video.
    Rationale: Based on the scenario, the agent should prioritize path efficiency and time efficiency, assigning 0.4 for each. Other objectives are not mentioned, assigning (1-0.4-0.4)/3=0.067 for each objective.
    Answer: [0.4,0.4,0.067,0.067,0.067]
6. Scenario: I will have a home party this week, but can't find where I put the vase to put on the table. I want to find it surely by today. I have enough time, so I just want the robot to find it.
    Rationale: Based on the scenario, the agent should prioritize house exploration the most, assigning 0.6. Object exploration is also important, assigning 0.3. Other objectives are not mentioned, assigning (1-0.6-0.3)/3=0.033 for each objective.
    Answer: [0.033,0.033,0.6,0.033,0.3]

Scenario: {}
Rationale: 
Answer: 
"""
\end{spverbatim}
}

The test instructions are as follows:

\begin{spverbatim}
instructions = [
"I'm getting late to an important meeting. I need to quickly find an object in the house.",
"I want to check an appliance in the house while I'm away, but I forgot to charge the robot last night. It seems that the robot that has a limited battery life, so I don't want the robot to waste time while looking into unnecessary regions.",
"I just moved in and want to find which furniture or object is located while inspecting the layout of the house as a video.",
"After rearranging the house, I can't remember where certain objects were placed. I want to find a specific object, while also inspecting other areas to confirm the new arrangement.",
"My house has lots of valuable and fragile artifacts. I want to find a special-edition item among those.",
"It's in the midnight, and I want to find an object in the house without making any loud noise. I don't want to disturb my child who is sleeping."
]
\end{spverbatim}

\begin{table*}[ht!]
\resizebox{\textwidth}{!}{
\begin{tabular}{cclcccccccccc}
\Xhline{2\arrayrulewidth}
{\textbf{Method}} & \textbf{Multi-Objective} & \multicolumn{2}{c}{\textbf{\begin{tabular}[c]{@{}c@{}}Prioritized\\ Objective\end{tabular}}} & \multicolumn{1}{c}{{\textbf{Success}}} & \multicolumn{1}{c}{{\textbf{SPL}}} & \multicolumn{1}{c}{{\textbf{\begin{tabular}[c]{@{}c@{}}Distance\\ to Goal\end{tabular}}}} & \multicolumn{1}{c}{{\textbf{\begin{tabular}[c]{@{}c@{}}Episode\\ Length\end{tabular}}}} & \multicolumn{5}{c}{\textbf{Sub Rewards $\uparrow$}} \\
 &  &  &  & $\uparrow$ & $\uparrow$ & $\downarrow$ & $\downarrow$ & Time Efficiency & Path Efficiency & House Exploration & Object Exploration & Safety \\ \hline
\multirow{6}{*}{\begin{tabular}[c]{@{}c@{}}Promptable\\ Behaviors\\ (Ours)\end{tabular}} & \multirow{6}{*}{\begin{tabular}[c]{@{}c@{}}Single-Policy\end{tabular}} & \multicolumn{1}{|c|}{\graytext{a}} & - & 0.470 & 0.190 & 1.632 & 182.540 & -1.825 & 1.179 & 3.370 & 2.095 & -3.162 \\
 &  & \multicolumn{1}{|c|}{\graytext{b}} & Time Efficiency & 0.410 & 0.185 & 1.736 & 156.470 & \cellcolor{LightCyan}\textbf{-1.565} & 0.959 & 3.112 & 2.066 & -1.500 \\
 &  & \multicolumn{1}{|c|}{\graytext{c}} & Path Efficiency & 0.420 & 0.164 & 1.656 & 199.330 & -1.993 & \cellcolor{LightCyan}\textbf{1.279} & 3.452 & 2.131 & -1.488 \\
 &  & \multicolumn{1}{|c|}{\graytext{d}} & House Exploration & \textbf{0.480} & 0.184 & 1.558 & 183.340 & -1.833 & 1.193 & \cellcolor{LightCyan}\textbf{3.681} & 2.125 & -1.885 \\
 &  & \multicolumn{1}{|c|}{\graytext{e}} & Object Exploration & \textbf{0.480} & 0.207 & 1.643 & 182.070 & -1.821 & 1.246 & 3.534 & \cellcolor{LightCyan}\textbf{2.137} & -1.573 \\ 
 &  & \multicolumn{1}{|c|}{\graytext{f}} & Safety & 0.450 & 0.175 & 1.629 & 156.770 & -1.568 & 1.247 & 3.122 & 2.098 & \cellcolor{LightCyan}\textbf{-1.485} \\
\hline 
\Xhline{2\arrayrulewidth}
\end{tabular}}
\caption{\textbf{Performance in RoboTHOR ObjectNav.} We evaluate each method in the validation set with five different configurations of objective prioritization: uniform reward weights among all objectives and prioritizing a single objective 10 times as much as other objectives. Sub-rewards for each objective are accumulated during each episode, averaged across episodes, and then normalized using the mean and variance calculated across all methods. Colored cells indicate the highest values in each sub-reward column.}\label{tab:robothor-objectnav}
\end{table*}

\begin{table*}[t!]
\footnotesize
\resizebox{\textwidth}{!}{
\begin{tabular}{ccclccccccc}
\hline
{\textbf{Method}} & \textbf{Multi-Objective} & \multicolumn{2}{c}{\textbf{\begin{tabular}[c]{@{}c@{}}Prioritized\\ Objective\end{tabular}}} & \multicolumn{1}{c}{{\textbf{Success}}} & \multicolumn{1}{c}{{\textbf{PLOPL}}} & \multicolumn{1}{c}{{\textbf{\begin{tabular}[c]{@{}c@{}}Distance\\ to Furthest\end{tabular}}}} & \multicolumn{1}{c}{{\textbf{\begin{tabular}[c]{@{}c@{}}Episode\\ Length\end{tabular}}}} & \multicolumn{3}{c}{\textbf{Sub Rewards $\uparrow$}} \\
 &  &  &  & $\uparrow$ & $\uparrow$ & $\downarrow$ & $\downarrow$ & Time Efficiency & House Exploration & Safety\\ \hline
\multirow{4}{*}{\begin{tabular}[c]{@{}c@{}}Promptable\\ Behaviors\\ (Ours)\end{tabular}} & \multirow{4}{*}{\begin{tabular}[c]{@{}c@{}}Single-Policy\end{tabular}} & \multicolumn{1}{|c|}{\graytext{a}} & - & 0.876 & 0.950 & 3.360 & 57.560 & 0.784 & 0.629 & 0.688 \\
 &  & \multicolumn{1}{|c|}{\graytext{b}} & Time Efficiency & 0.890 & 0.946 & 3.501 & 56.030 & \cellcolor{LightCyan}\textbf{0.920} & 0.053 & 0.367 \\
 &  & \multicolumn{1}{|c|}{\graytext{c}} & House Exploration & \textbf{0.902} & 0.954 & 3.017 & 59.320 & 0.627 & \cellcolor{LightCyan}\textbf{0.907} & 0.932 \\
 &  & \multicolumn{1}{|c|}{\graytext{d}} & Safety & 0.854 & 0.946 & 3.630 & 63.890 & 0.220 & 0.730 & \cellcolor{LightCyan}\textbf{0.934} \\ \hline
\end{tabular}}\vspace{-0.1cm}
\caption{\textbf{Performance in RoboTHOR FleeNav.} We evaluate each method in the validation set with four different configurations of objective prioritization: uniform reward weight across all objectives and prioritizing a single objective 10 times as much as other objectives. Sub-rewards for each objective are accumulated during each episode, averaged across episodes, and then normalized using the mean and variance calculated across all methods. Colored cells indicate the highest values in each sub-reward column.}\label{tab:robothor-fleenav}
\end{table*}

\noindent For FleeNav, we use the following template:

{\footnotesize
\begin{spverbatim}
"""
In the flee navigation task in ProcTHOR, an agent is placed within a simulated environment with the aim to move as far away as possible from its starting position. The task tests the agent's ability to maximize the distance from its initial location while considering various objectives that determine its behavior.

Objectives are:
1. Time Efficiency: Aim to find the target object using as few steps as possible.
2. House Exploration: Strive to explore the house thoroughly. This involves checking many different areas/rooms until you find the farthest point from the agent’s initial location.
3. Safety: Navigate while avoiding obstacles and areas where you could get trapped or stuck.

Given an instruction, I want to know the weights over the four objectives (Time Efficiency, House Exploration, Safety).
The weights should be spiked, meaning that the weight of the most important objective should be much higher than the weight of the least important objective.

Here are some examples.
1. Instruction: Prioritize getting as far from your starting point as possible, regardless of the number of steps.
   Rationale: The instruction describes that time efficiency is the least important, assigning 0.2. House exploration and safety are not mentioned but should be important than time efficiency, so they are assigned 0.4 each.
   Answer: [0.2,0.4,0.4]
2. Instruction: Explore the environment thoroughly while avoiding colliding to walls and obstacles.
   Rationale: The instruction describes that house exploration is the most important, assigning 0.5. Safety is the second priority, assigning 0.4. Time efficiency is not mentioned, so it is assigned 0.1.
   Answer: [0.1,0.5,0.4]
3. Instruction: Your main goal is to explore while distancing from the start.
   Rationale: The instruction describes that house exploration is the most important, assigning 0.6. Safety and time efficiency are not mentioned, so they are assigned 0.2 each.
   Answer: [0.2,0.6,0.2]
4. Instruction: Safety is key. Move away, but avoid any and all obstacles.
   Rationale: The instruction describes that safety is the most important, assigning 0.6. House exploration and time efficiency are not mentioned, so they are assigned 0.2 each.
   Answer: [0.2,0.2,0.6]
5. Instruction: Avoid taking too many steps.
   Rationale: The instruction describes that time efficiency is the most important, assigning 0.6. House exploration and safety are not mentioned, so they are assigned 0.2 each.
   Answer: [0.6,0.2,0.2]
6. Instruction: Prioritize safety first, then exploration.
   Rationale: The instruction describes that safety is the most important, assigning 0.6. House exploration is the second important, assigning 0.4. Time efficiency is not mentioned, so it is assigned 0.0.
   Answer: [0.0,0.4,0.6]

Instruction: {}
Rationale: 
Answer: 
"""
\end{spverbatim}
}

\begin{figure*}[ht!]{\centering
\captionsetup[subfigure]{justification=centering}
    \subfloat[Time Efficiency vs.\\ House Exploration]{
            \includegraphics[width=0.25\linewidth]{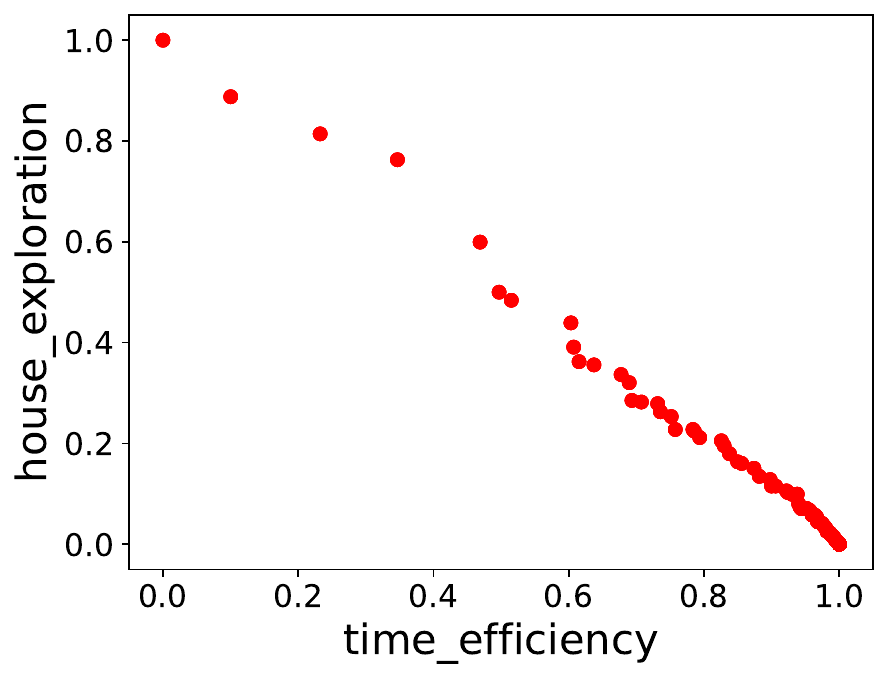}
            }\centering
    \subfloat[Path Efficiency vs.\\ House Exploration]{
            \includegraphics[width=0.25\linewidth]{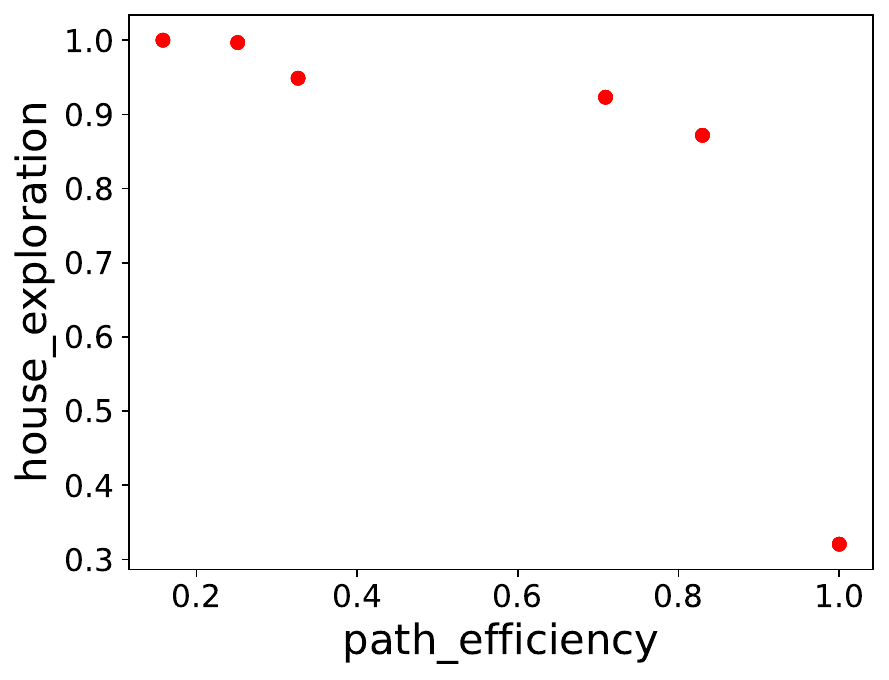}
            }\centering
    \subfloat[Path Efficiency vs.\\ Safety]{
            \includegraphics[width=0.25\linewidth]{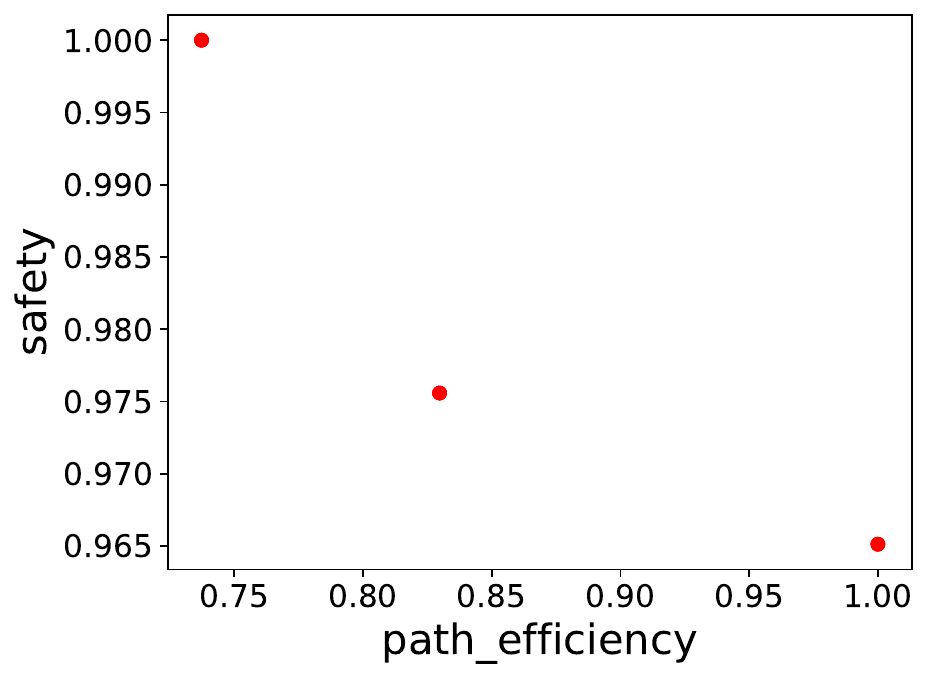}
            }\centering
    \subfloat[Time Efficiency vs.\\ Path Efficiency]{
            \includegraphics[width=0.25\linewidth]{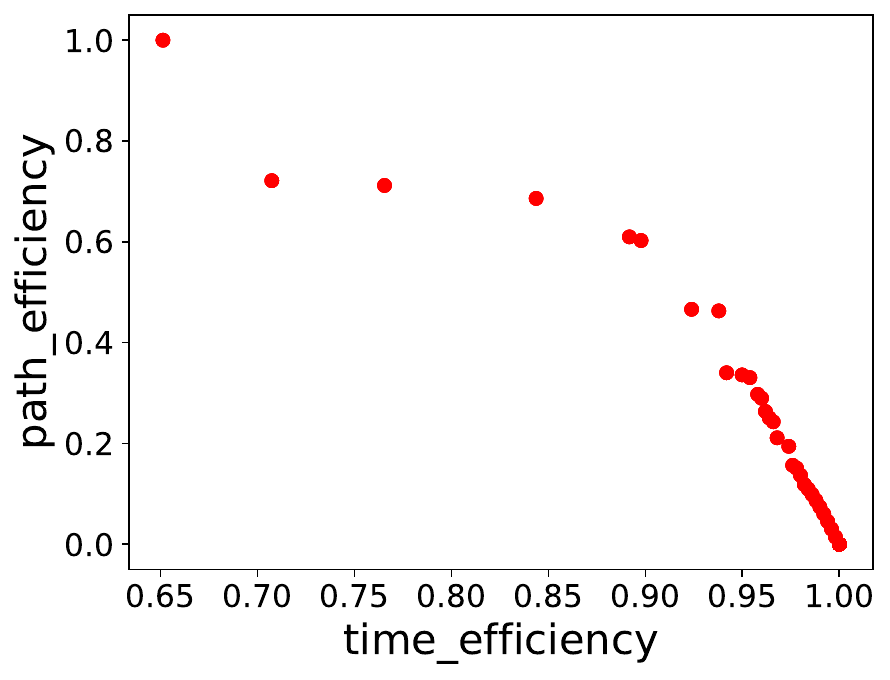}
            }\centering
    \caption{\textbf{Pareto front plots in ProcTHOR ObjectNav.} We plot the Pareto front on four different combinations of objective pairs in ProcTHOR ObjectNav. Data samples are generated with five different reward configurations, each prioritizing a single objective $4$ times more than other objectives. Each data point corresponds to a sample with rewards on two objectives. Rewards are normalized for each objective.}\label{fig:pareto-front-objectnav}
    }
\end{figure*}

\section{Additional Results and Analyses}\label{sec:experiment-results}
In this section, we provide experiment results for RoboTHOR ObjectNav and FleeNav. Additionally, we visualize trajectories for numerous episodes and analyze those qualitatively. Lastly, detailed analyses for reward weight prediction experiments are provided.
\subsection{RoboTHOR Experiment Results}
\noindent \textbf{RoboTHOR ObjectNav.} Results in Table~\ref{tab:robothor-objectnav} show that the proposed \textit{Promptable Behaviors} outputs different agent behaviors and performances based on the prioritization of objectives. Regarding the general performance, prioritizing house exploration or object exploration shows the highest success rate of $48.0\%$. When time efficiency is prioritized (row b in Table~\ref{tab:robothor-objectnav}), the agent achieves $\%$ higher time efficiency reward compared to the case when no objective is prioritized (row a in Table~\ref{tab:robothor-objectnav}). Similarly, prioritizing a single objective improves the corresponding sub-reward.

\noindent \textbf{RoboTHOR FleeNav.} Results in Table~\ref{tab:robothor-fleenav} also show that our method can effectively prompt agent behaviors by adjusting the reward weights across multiple objectives. Among all reward weight configurations, prioritizing house exploration shows the highest success rate. When safety is prioritized, the agent shows a higher safety reward compared to all other reward weight configurations. 

\subsection{Pareto Front Analysis}

\begin{figure}[h!]{\centering
    \includegraphics[width=0.7\linewidth]{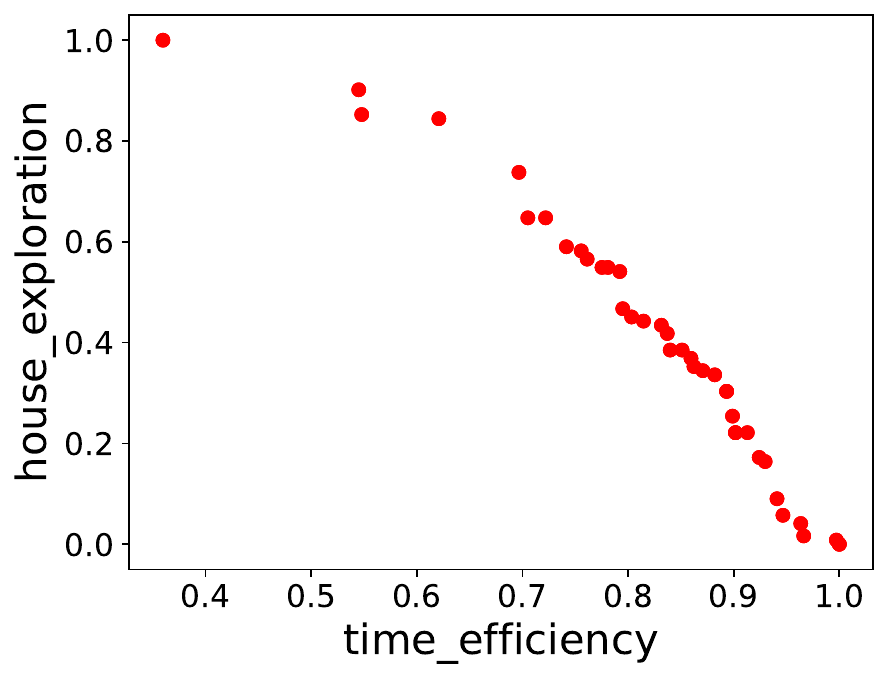}
    \centering
    \caption{\textbf{Pareto front comparing time efficiency and house exploration in ProcTHOR FleeNav.} We plot the Pareto front in ProcTHOR FleeNav comparing time efficiency and house exploration. Data samples are generated with three different reward configurations, each prioritizing a single objective $3$ times more than other objectives. Each axis denotes the reward for the corresponding objective and the rewards are normalized following Table~\ref{fig:pareto-front-objectnav}.}\label{fig:pareto-front-fleenav}
    }
\end{figure}

\noindent \textbf{Efficiency and house exploration are conflicting objectives.} Pareto front plots in Figure~\ref{fig:pareto-front-objectnav} (a) and (b), and Figure~\ref{fig:pareto-front-fleenav} show that time efficiency and path efficiency are conflicting objectives. Comparing Figure~\ref{fig:pareto-front-objectnav} (a) and (b), path efficiency could be interpreted as more conflictive against house exploration since the area below the Pareto front in (b) is larger than the area in (a). 
Interestingly, Figure~\ref{fig:pareto-front-objectnav} (d) illustrates a nearly convex Pareto front when time efficiency and path efficiency are compared. This implies that although time efficiency and path efficiency both aim \textit{efficiency} on the agent's behavior, a notable difference exists between these two objectives.
Comparing safety with path efficiency in Figure~\ref{fig:pareto-front-objectnav} (c), the Pareto front is illustrated as a concave curve. Comparing the curve with the Pareto front in Figure~\ref{fig:pareto-front-objectnav} (b), path efficiency is more conflictive with house exploration than safety.

\subsection{Trajectory Visualizations}
We visualize trajectories to compare how agent behaviors are different when different objectives are prioritized.

\begin{figure*}[ht!]{\centering
    \subfloat[\centering Time Efficiency vs. House Exploration]{
            \includegraphics[width=1\linewidth]{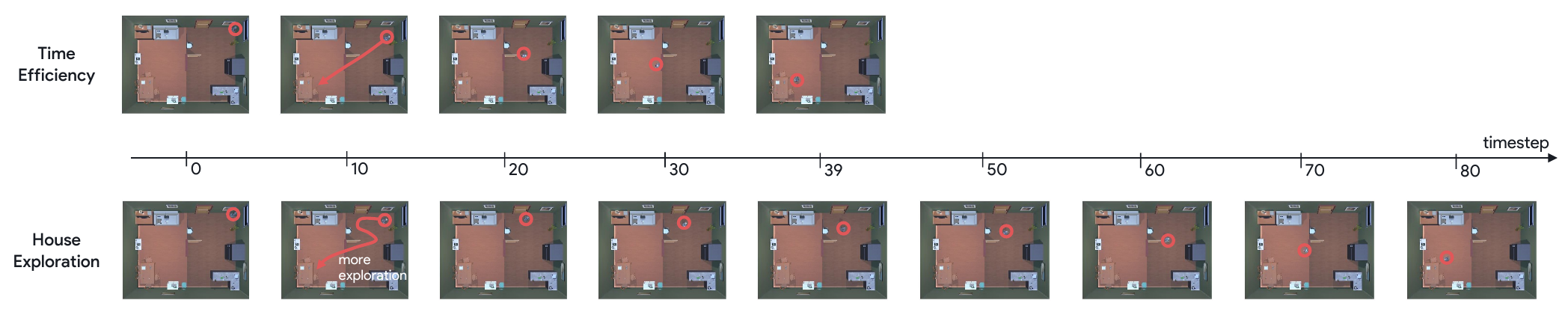}
            }\centering\\\vspace{0.4cm}
    \subfloat[\centering Path Efficiency vs. House Exploration]{
            \includegraphics[width=1\linewidth]{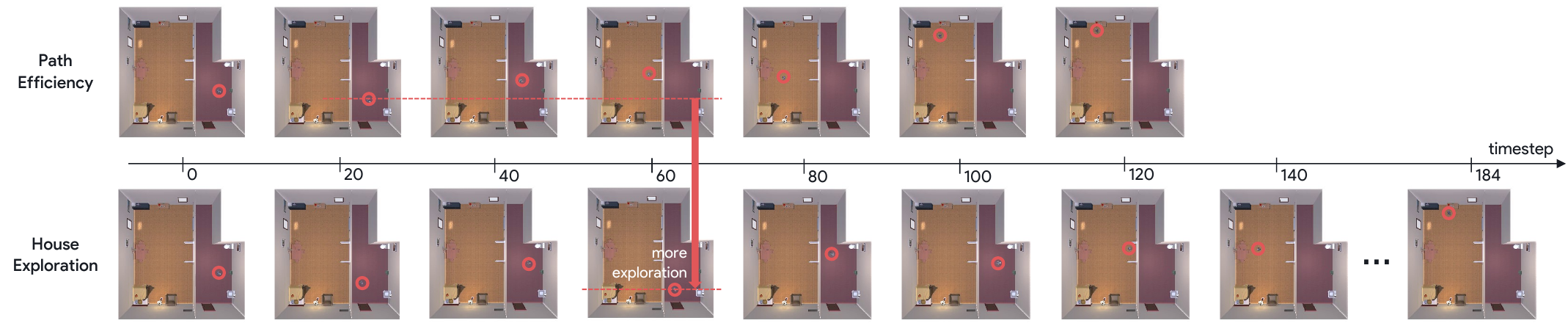}
            }\centering\\\vspace{0.4cm}
    \subfloat[\centering Safety vs. House Exploration]{
            \includegraphics[width=1\linewidth]{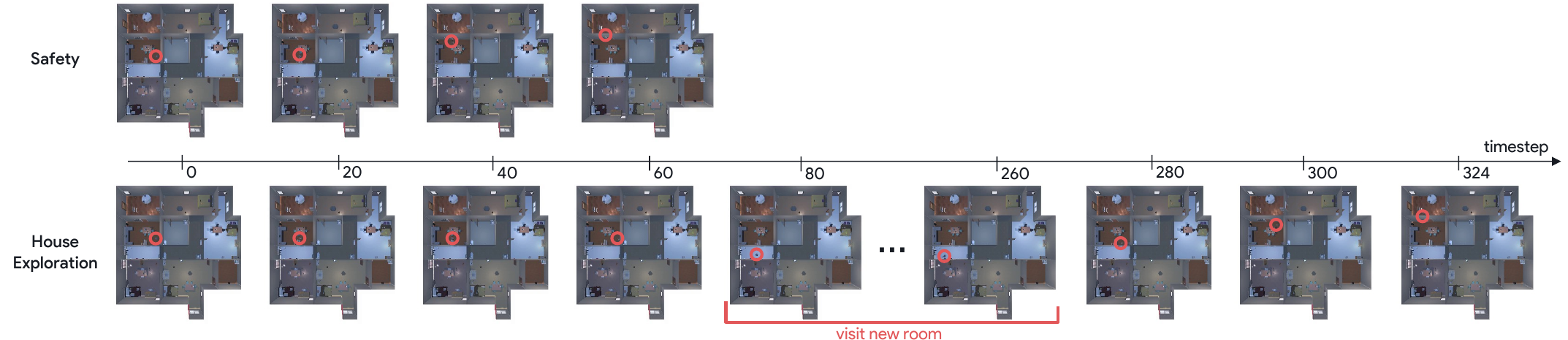}
            }\centering
    \caption{\textbf{Comparing House Exploration with other objectives.} We compare agent trajectories when house exploration is prioritized with when one of three objectives (time efficiency, path efficiency, and safety) is prioritized.}\label{fig:compare-house-exploration}
    }
\end{figure*}

\begin{figure*}[ht!]{\centering
    \subfloat[\centering House 18 Episode - Find AlarmClock]{
            \includegraphics[width=1\linewidth]{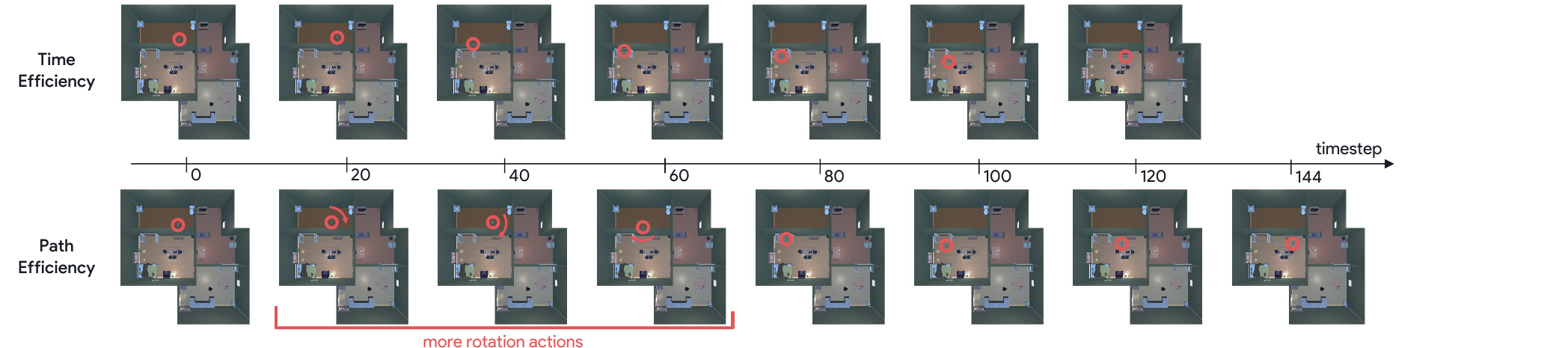}
            }\centering\\\vspace{0.2cm}
    \subfloat[\centering House 82 Episode - Find HousePlant]{
            \includegraphics[width=1\linewidth]{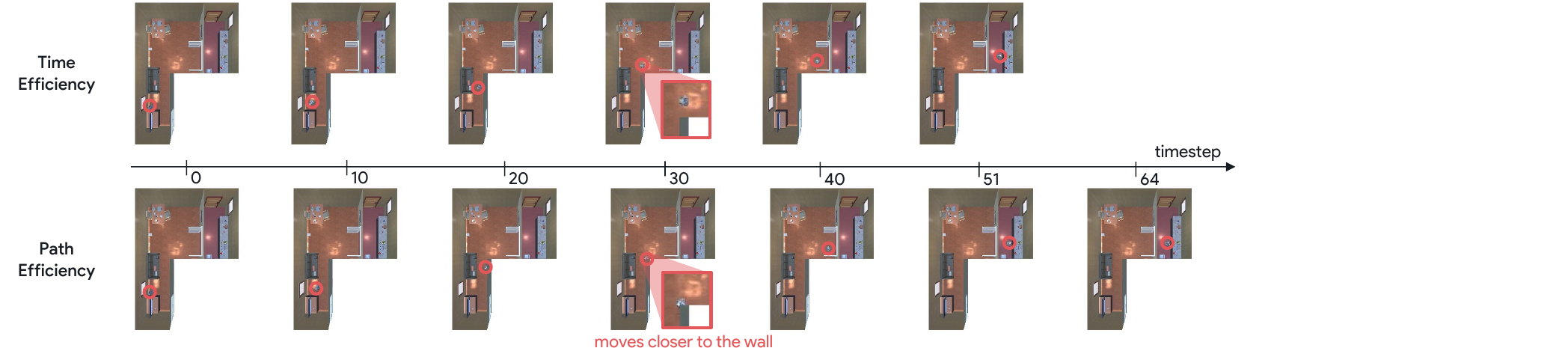}
            }\centering
    \caption{\textbf{Time Efficiency vs. Path Efficiency.} We compare agent trajectories when time efficiency is prioritized and when path efficiency is prioritized in the same episode. Both episodes in (a) and (b) show that prioritizing time efficiency encourages the agent to end the episode faster than prioritizing path efficiency. Prioritizing path efficiency showed interesting behaviors, such as performing more rotation actions in (a) and moving closer to the wall in (b), compared to prioritizing time efficiency.}\label{fig:time-efficiency-vs-path-efficiency}\vspace{1cm}
    }
\end{figure*}

\begin{figure*}[h!]{\centering
            \includegraphics[width=1\linewidth]{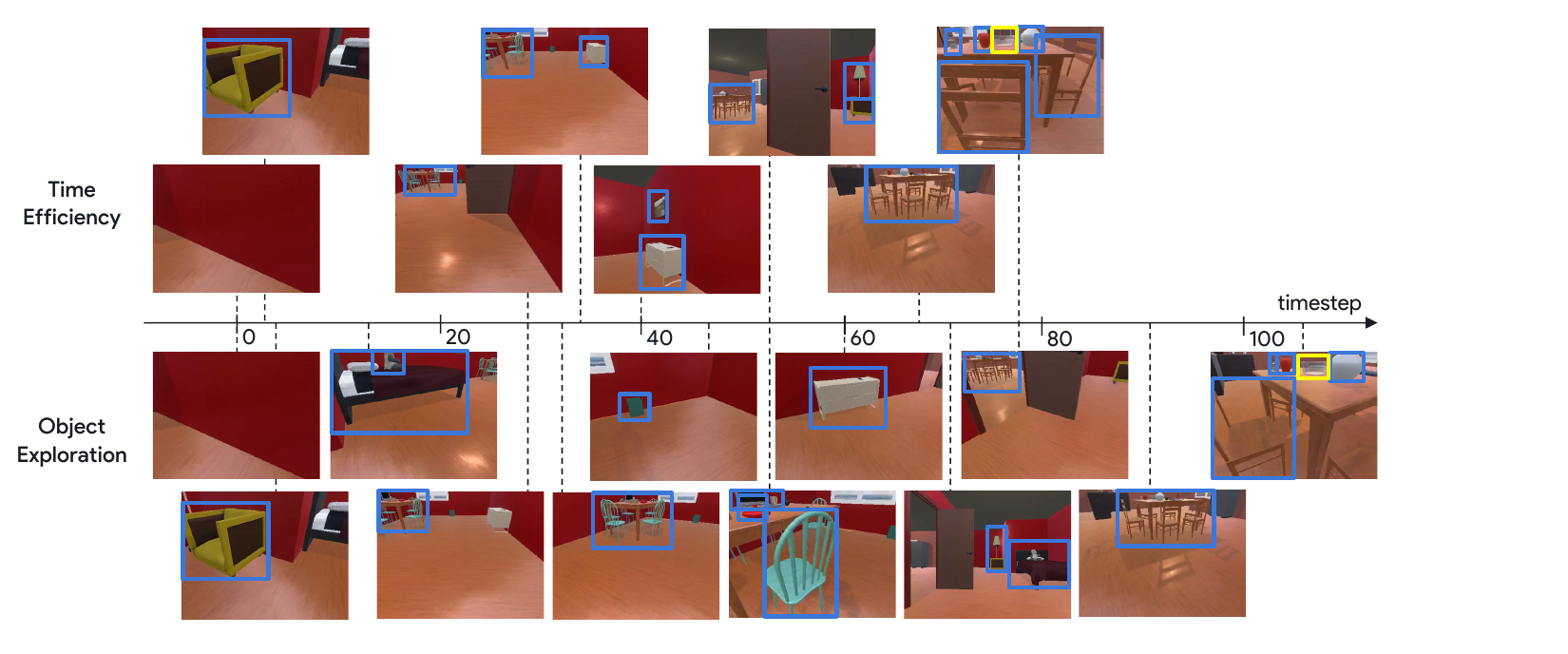}
    \caption{\textbf{Time Efficiency vs Object Exploration.} We compare agent trajectories when time efficiency is prioritized with when object exploration is prioritized. Blue boxes denote objects detected in each image and yellow box describes the bounding box of the target object.}\label{fig:compare-object-exploration}\vspace{1cm}
    }
\end{figure*}

\begin{figure*}[ht!]{\centering
    \subfloat[\centering Time Efficiency vs. Safety]{
            \includegraphics[width=1\linewidth]{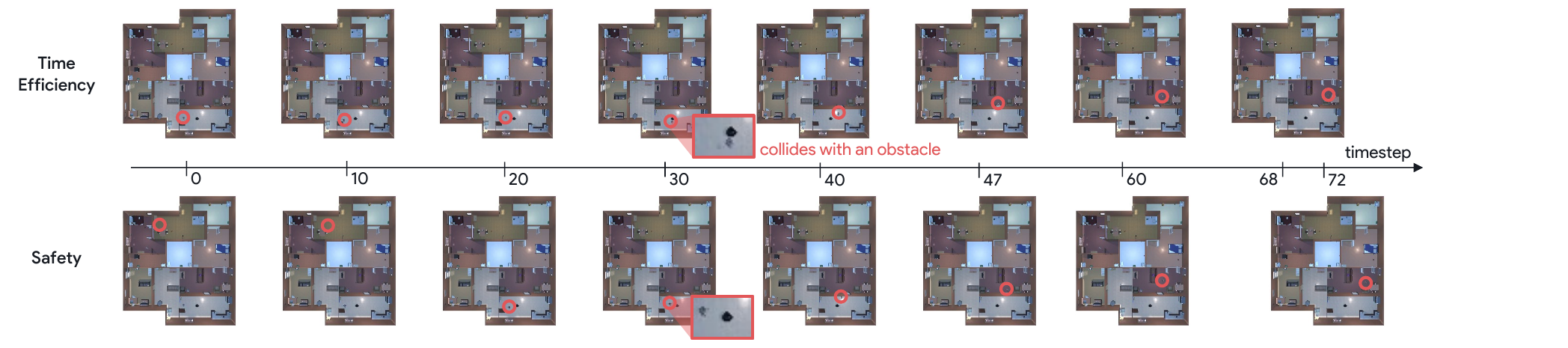}
            }\centering\\\vspace{2cm}
    \subfloat[\centering Path Efficiency vs. Safety]{
            \includegraphics[width=1\linewidth]{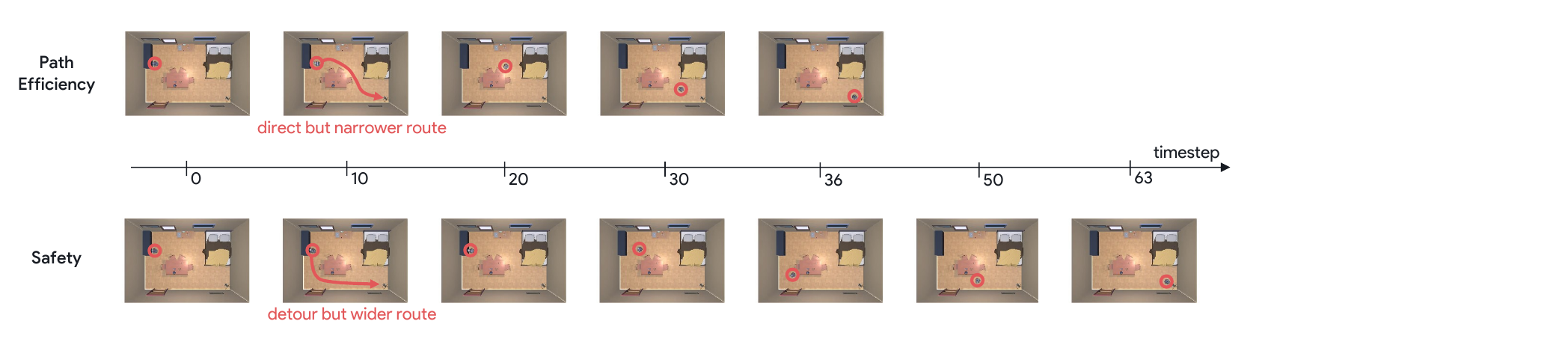}
            }\centering\\\vspace{2cm}
    \subfloat[\centering House Exploration vs. Safety]{
            \includegraphics[width=1\linewidth]{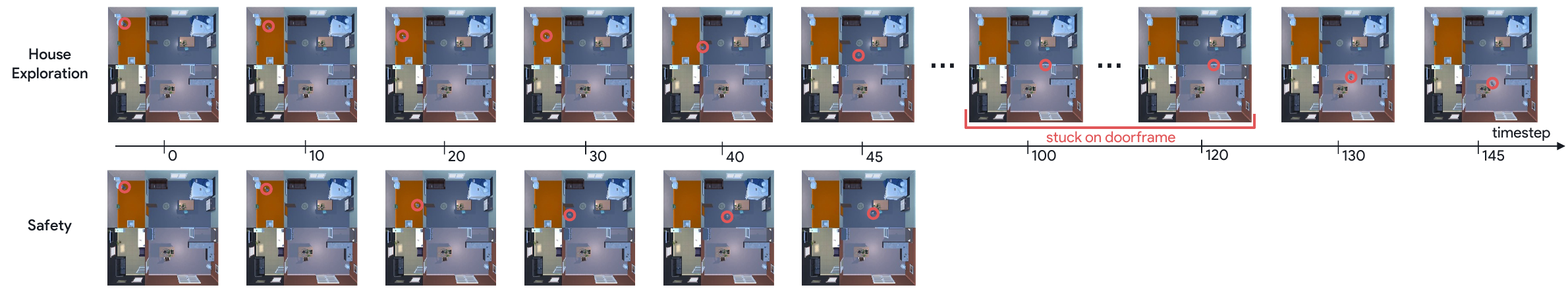}
            }\centering
    \caption{\textbf{Comparing Safety with other objectives.} We compare agent trajectories when safety is prioritized with when one of three objectives (time efficiency, path efficiency, and house exploration) is prioritized. (a) The agent collides with an obstacle when time efficiency is prioritized the most. When safety is prioritized, the agent does not collide with any object and smoothly finds the target object. (b) When path efficiency is prioritized, the agent follows a direct but narrower route compared to the case when safety is prioritized. (c) In this episode, multiple objects corresponding to the target object category exist in the house. The agent gets stuck on a doorframe when house exploration is prioritized, while the agent smoothly passes the door when safety is prioritized.}\label{fig:compare-safety}
    }\vspace{3cm}
\end{figure*}

\noindent \textbf{Time efficiency saves time.} Figure~\ref{fig:compare-house-exploration} (a), Figure~\ref{fig:time-efficiency-vs-path-efficiency}, Figure~\ref{fig:compare-safety} (a), and Figure~\ref{fig:compare-object-exploration} shows that within the same episode, the agent finds the target object faster when time efficiency is prioritized than other objectives. In Figure~\ref{fig:compare-house-exploration} (a), the agent that prioritizes time efficiency directly gets out of a room and finds the target object earlier than the case when house exploration is prioritized. Even comparing time efficiency against path efficiency in Figure~\ref{fig:time-efficiency-vs-path-efficiency}, prioritizing time efficiency encourages the agent to find the target object earlier. Notably, we found that time efficiency and path efficiency are correlated but different in some cases. As shown in Figure~\ref{fig:time-efficiency-vs-path-efficiency} (a), prioritizing path efficiency often causes more rotation actions than prioritizing time efficiency. This could be because changing the agent's orientation can help the agent find shorter paths. Figure~\ref{fig:time-efficiency-vs-path-efficiency} describes another difference: prioritizing path efficiency encourages the agent to move closer to a corner of the wall. This could be due to the nature of calculating the geodesic shortest path since the shortest path is usually a series of straight-line segments that connect corner places in the environment. In Figure~\ref{fig:compare-object-exploration}, prioritizing time efficiency results in a shorter path than prioritizing object exploration. These results imply that the agent learns to finish the episode faster when time efficiency is prioritized.

\noindent \textbf{House exploration enables the agent to explore the house thoroughly.} Figure~\ref{fig:compare-house-exploration} visualizes how agent trajectories are changed when house exploration is prioritized. In Figure~\ref{fig:compare-house-exploration} (a), compared to the trajectory that prioritizes time efficiency, the agent prioritizing house exploration explores the room more before getting out of the room. Similarly, Figure~\ref{fig:compare-house-exploration} (b) illustrates an episode where the agent explores more when house exploration is prioritized than when path efficiency is prioritized. In Figure~\ref{fig:compare-house-exploration} (c), the agent visited more rooms when house exploration is prioritized, compared to the case when safety is prioritized. These qualitative results demonstrate the effectiveness of \textit{Promptable Behaviors} that prioritizing house exploration encourages the agent to explore the house thoroughly.

\noindent \textbf{Object exploration encourages the agent to inspect more objects.} Figure~\ref{fig:compare-object-exploration} visualizes the RGB observations with objects detected in each trajectory. The figure shows that prioritizing object exploration makes the agent inspect objects more in detail, often encouraging getting close to the observed objects. For instance, the agent observed a bed, a chair, and a cabinet closer than the trajectory that prioritizes time efficiency. Also, the agent changed the orientation of the view by performing \texttt{LookUp} and \texttt{LookDown} actions before leaving the room when object exploration is prioritized. In contrast, the agent prioritizing time efficiency did not perform any \texttt{LookUp} or \texttt{LookDown} actions and got out of the room in an earlier timestep.

\noindent \textbf{Safety avoids colliding with objects, doors, and walls.} Figure~\ref{fig:compare-safety} describes agent trajectories that prioritize safety or one of three objectives (time efficiency, path efficiency, and house exploration). In Figure~\ref{fig:compare-safety} (a), prioritizing time efficiency resulted in a collision between the agent and an obstacle in the middle of a room, while the safety-prioritized agent did not collide with any objects. In Figure~\ref{fig:compare-safety} (b), the agent prioritizing safety took a detour but a wider route to move towards the target object, while the agent prioritizing path efficiency followed a direct and narrower route. In Figure~\ref{fig:compare-safety}, the agent prioritizing house exploration got stuck on a doorframe for 20 timesteps, while the safety-prioritized agent did not get stuck in any places. These results imply that the agent learns to avoid colliding with objects and places that could potentially make the agent get stuck.

\subsection{Reward Weight Prediction}

\begin{table}[h]
\resizebox{0.48\textwidth}{!}{
\begin{tabular}{llrcc}
\hline
\multicolumn{3}{c}{\textbf{Weight Prediction Methods}} & \multirow{2}{*}{\textbf{Acc.} $\uparrow$} & \multirow{2}{*}{\textbf{GGI} $\uparrow$} \\
Input & Model & $N$ &  &  \\
\hline
Human Demonstrations & - & 1 & 0.707 & 0.347 \\
\cmidrule(r){1-3}
\multirow{12}{*}{\begin{tabular}[c]{@{}c@{}}Preference Feedback\end{tabular}} & \multirow{4}{*}{\begin{tabular}[c]{@{}l@{}}Pairwise\\ Comparison\\ (M=1)\end{tabular}} & 10 & 0.369 & 0.800 \\
 &  & 20 & 0.356 & 0.800 \\
 &  & 50 & 0.358 & 0.800 \\
 &  & 100 & 0.505 & 0.800 \\
 &  & 200 & 0.587 & 0.800 \\
 &  & 500 & 0.897 & 0.800 \\
 \cmidrule(r){2-3}
 & \multirow{3}{*}{\begin{tabular}[c]{@{}l@{}}Group\\ Comparison\\(M=2)\end{tabular}} & 5 & 0.689 & 0.626 \\
 &  & 10 & 0.793 & 0.618 \\
 &  & 25 & \textbf{0.935} & 0.657 \\
 \cmidrule(r){2-3}
 & \multirow{3}{*}{\begin{tabular}[c]{@{}l@{}}Group\\ Comparison\\(M=5)\end{tabular}} & 2 & 0.722 & 0.634 \\
 &  & 4 & 0.682 & 0.762 \\
 &  & 10 & 0.862 & 0.641 \\
\cmidrule(r){1-3}
\multirow{4}{*}{\begin{tabular}[c]{@{}c@{}}Language Instructions\end{tabular}} &  ChatGPT & 1 & 0.530 & 0.388 \\
 &  \ \ w/ ICL & 1 & 0.529 & 0.379 \\
 & \ \ w/ CoT & 1 & 0.614 & 0.391 \\
 &  \ \ w/ ICL + CoT & 1 & 0.482 & 0.347 \\
\hline
\end{tabular}}
\caption{\textbf{Comparison of Three Weight Prediction Methods in ProcTHOR ObjectNav.} We predict the optimal reward weight vector from the collected human demonstrations, preference feedback on trajectory comparisons, and language instructions. For trajectory comparison methods, the total number of observed trajectories is $2NM$.}\label{tab:weight-prediction-results}
\end{table}

Table~\ref{tab:weight-prediction-results} shows the full results of reward weight prediction experiments. 
Group trajectory comparison when $M=5$ aligns with the probabilistic guarantee of group preference with an error less than $\delta=0.01$, resulting in $86.2\%$ prediction accuracy with only $10$ group comparisons. The results show that group trajectory comparison with $M=5$ and $10$ feedback shows a similar performance of pairwise trajectory comparison with $500$ feedback. This implies that group trajectory comparison effectively reduces human effort on providing preference feedback. 

\section{Ablation Study}\label{sec:ablation-study}
\begin{figure}[ht!]{\centering
    \subfloat[\centering Success Rate]{
            \includegraphics[width=0.6\linewidth]{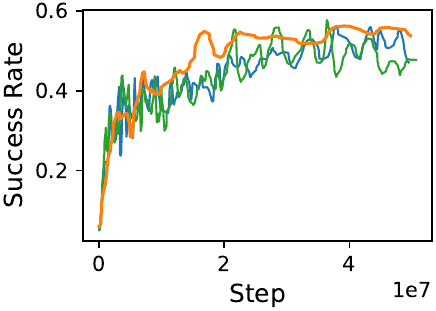}
            }\centering\\
    \subfloat[\centering Distance to Target]{
            \includegraphics[width=0.6\linewidth]{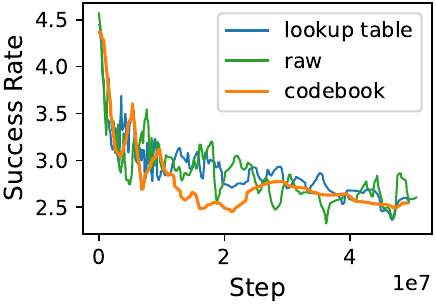}
            }\centering

    \caption{\textbf{Training Curves with Different Weight Embedding Methods in ProcTHOR ObjectNav.} We compare three weight embedding methods: raw, lookup table, and codebook.}\label{fig:ablation-weight-embed-type}
    }
\end{figure}

\begin{table}[t!]
\resizebox{0.48\textwidth}{!}{
\begin{tabular}{cccccc} 
\Xhline{2\arrayrulewidth}
{\textbf{Method}} & \textbf{Input Type} & \textbf{Input Range} & \textbf{Success $\uparrow$} & \textbf{SPL $\uparrow$} \\ 
\hline
Raw & Conti. & $[0.0,1.0]$ & 0.393 \scalebox{.7}{$\pm$} 0.064 & 0.259 \scalebox{.7}{$\pm$} 0.020 \\ 
Lookup Table & Discrete & $[0,10]$ & 0.409 \scalebox{.7}{$\pm$} 0.007 & 0.284 \scalebox{.7}{$\pm$} 0.003 \\
\rowcolor{Gray}Codebook (Ours) & Conti. & $[0.0,1.0]$ & 0.396 \scalebox{.7}{$\pm$} 0.027 & 0.298 \scalebox{.7}{$\pm$} 0.022 \\ 
\hline 
\Xhline{2\arrayrulewidth}
\end{tabular}}
\caption{\textbf{Comparison of Reward Weight Encoding Methods.} We evaluate each method in the validation set with 11 different configurations of objective prioritization: uniform reward weights among all objectives, prioritizing a single objective 4 times as much as other objectives, and prioritizing a single objective 10 times as much as other objectives.}\label{tab:ablation-reward-weight-encoder}
\end{table}

\noindent\textbf{Is Codebook Effective in MORL?}\label{ablation:codebook}
We examine the effectiveness of the codebook module for encoding reward weight vector in MORL. The training curves, shown in Figure~\ref{fig:ablation-weight-embed-type}, illustrate that using codebook results in a more stable performance compared to using a lookup table extended from \cite{singh2022ask4help} and using raw reward weight vectors without encoding. After training with each reward weight encoding method for $50M$ steps, we evaluate the agent with $11$ different reward configurations, as described in Table~\ref{tab:ablation-reward-weight-encoder}. The proposed encoding method using codebook improved the average success rate by $0.76\%$ and the standard deviation of success rate by $57.8\%$ compared to using raw reward weight vectors without encoding. This indicates that the training process is stabilized through the codebook module. While the lookup table method did exhibit a higher average success rate and SPL than using codebook, we address that using integer weights to represent reward configurations has a critical limitation: ambiguity in the reward weight prediction phase. For instance, different integer weight vectors like $[4,1,1,1,1]$ and $[8,2,2,2,2]$ may represent identical preferences, leading to multiple potential solutions and complicating the inference of optimal reward weights from human preferences. In contrast, our method represents human preferences using continuous real-values weight vectors, such as $[0.5,0.125,0.125,0.125,0.125]$.

\end{document}